\documentclass{article}
\usepackage{microtype}
\usepackage{graphicx}
\usepackage{subfigure}
\usepackage{booktabs} 
\usepackage{mathrsfs}

\usepackage{nicefrac}
\usepackage{xcolor}
\usepackage{amsmath,amssymb,amsthm}
\usepackage{thm-restate}

\theoremstyle{definition}

\usepackage{enumitem}
\setlist[enumerate]{leftmargin=0.5cm,topsep=0pt,itemsep=-2pt}
\setlist[itemize]{leftmargin=0.5cm,topsep=0pt,itemsep=-2pt}

\usepackage{mathrsfs}
\newcommand{\probspace}{\mathscr{P}}
\DeclareMathOperator{\sg}{sg}
\DeclareMathOperator{\diag}{diag}
\DeclareMathOperator{\spanop}{span}

\usepackage{mathtools}
\usepackage{hyperref}
\hypersetup{
    colorlinks=true,
    linkcolor=blue,
    citecolor=cyan,
}

\usepackage[capitalize]{cleveref}
\Crefname{assumption}{Assumption}{Assumptions}
\Crefname{equation}{Equation}{Equations}
\creflabelformat{equation}{(#2#1#3)}

\usepackage[accepted]{paper}

\icmltitlerunning{Understanding Self-Predictive Learning for Reinforcement Learning}

\begin{document}

\twocolumn[
\icmltitle{Understanding Self-Predictive Learning for Reinforcement Learning}

\icmlsetsymbol{equal}{*}

\begin{icmlauthorlist}
\icmlauthor{Yunhao Tang}{dm}
\icmlauthor{Zhaohan Daniel Guo}{dm}
\icmlauthor{Pierre Harvey Richemond}{dm}
\icmlauthor{Bernardo \'Avila Pires}{dm}
\icmlauthor{Yash Chandak}{su}
\icmlauthor{R\'emi Munos}{dm}
\icmlauthor{Mark Rowland}{dm}
\icmlauthor{Mohammad Gheshlaghi Azar}{dm}
\icmlauthor{Charline Le Lan}{ox}
\icmlauthor{Clare Lyle}{dm}
\icmlauthor{Andr\'as György}{dm}
\icmlauthor{Shantanu Thakoor}{dm}
\icmlauthor{Will Dabney}{dm}
\icmlauthor{Bilal Piot}{dm}
\icmlauthor{Daniele Calandriello}{dm}
\icmlauthor{Michal Valko}{dm}
\end{icmlauthorlist}

\icmlaffiliation{dm}{DeepMind}
\icmlaffiliation{su}{Stanford University}
\icmlaffiliation{ox}{University of Oxford}

\icmlcorrespondingauthor{Yunhao Tang}{robintyh@deepmind.com}

\icmlkeywords{Machine Learning, ICML}

\vskip 0.3in
]

\printAffiliationsAndNotice{\icmlEqualContribution} 

\begin{abstract}
    We study the learning dynamics of self-predictive learning for reinforcement learning, a family of algorithms that learn representations by minimizing the prediction error of their own future latent representations.
    Despite its recent empirical success, such algorithms have an apparent defect: trivial representations (such as constants) minimize the prediction error, yet it is obviously undesirable to converge to such solutions.
    Our central insight is that careful designs of the optimization dynamics are critical to learning meaningful representations. We identify that a faster paced optimization of the predictor and semi-gradient updates on the representation, are crucial to preventing the representation collapse. Then in an idealized setup, we show self-predictive learning dynamics carries out spectral decomposition on the state transition matrix, effectively capturing information of the transition dynamics. Building on the theoretical insights, we propose bidirectional self-predictive learning, a novel self-predictive algorithm that learns two representations simultaneously. We examine the robustness of our theoretical insights with a number of small-scale experiments and showcase the promise of the novel representation learning algorithm with large-scale experiments.
\end{abstract}

\section{Introduction}

Self-prediction is one of the fundamental concepts in reinforcement learning (RL). In value-based RL, temporal difference (TD) learning \citep{sutton1988learning} uses the value function prediction at the next time step as the prediction target for the current time step $V(x_t)\leftarrow R(x_t) + \gamma V(x_{t+1})$, a procedure also known as \emph{bootstrapping}. We can understand TD-learning as self-prediction specialized to value learning, where the value function makes predictions about targets constructed from itself.

Recently, the idea of self-prediction has been extended to representation learning with much empirical success~\citep{schwarzer2021dataefficient,guo2020bootstrap,guo2022byol}. In self-predictive learning, the aim is to learn a representation $\Phi$ jointly with a transition function $P$ which models the transition of representations in the latent space $\Phi(x_t)\rightarrow\Phi(x_{t+1})$, by minimizing the prediction error 
\begin{align*}
    \left\lVert P\left(\Phi(x_t)\right)-  \Phi(x_{t+1}) \right\rVert_2^2.
\end{align*}
Intuitively, minimizing the prediction error should encourage the algorithm to learn a compressed latent representation $\Phi(x)$ of the state $x$. However, despite the intuitive construct of the prediction error, there is no obvious theoretical justification why minimizing the error leads to meaningful representations at all. Indeed, the \emph{trivial} solution $\Phi(x_t)\equiv c$ for any constant vector $c$ minimizes the error but retains no information at all. Comparing self-predictive learning with value learning, the key difference lies in that the value function is \emph{grounded} in the immediate reward $R(x_t)$. In contrast, the prediction error that motivates self-predictive learning is apparently not grounded in concrete quantities in the environment. 

A number of natural questions ensue: how do we reconcile the apparent defect of the prediction error objective, with the empirical success of practical algorithms built on such an objective? What are the representations obtained by self-predictive learning, and are they useful for downstream RL? With obvious theory-practice conflicts in place, it is difficult to establish self-predictive learning as a principled approach to representation learning in general.

We present the first attempt at understanding self-predictive learning for RL, through a theoretical lens. In an idealized setting, we identify key elements to ensure that the self-predictive algorithm avoids collapse and learns meaningful representations. We make the following theoretical and algorithmic contributions.
\paragraph{Key algorithmic elements to prevent collapse.} We identify two key algorithmic components: (1) the two time-scale optimization of the transition function $P$ and representation $\Phi$; and (2) the semi-gradient update on $\Phi$, to ensure that the representation maintains its capacity throughout learning (\cref{sec:dynamics}). As a result, self-predictive learning dynamics does not converge to trivial solutions starting from random initializations.

\paragraph{Self-prediction as spectral decomposition.} With a few idealized assumptions in place, we show that the learning dynamics locally improves upon a trace objective that characterizes the information that the representations capture about the transition dynamics (\cref{sec:dynamics}). Maximizing this objective corresponds to spectral decomposition on the state transition matrix. This provides a partial theoretical understanding as to why self-predictive learning proves highly useful in practice.

\paragraph{Bidirectional self-predictive learning.} Based on the theoretical insights, we derive a novel self-predictive learning algorithm: bidirectional self-predictive learning (\cref{sec:double}). The new algorithm learns two representations simultaneously, based on both a forward prediction and a backward prediction. Bidrectional self-predictive learning enjoys more general theoretical guarantees compared to self-predictive learning, and obtains more consistent and stable performance as we validate both on tabular and deep RL experiments (\cref{sec:exp}).
 
\section{Background}

Consider a reward-free Markov decision process (MDP) represented as the tuple $\left(\mathcal{X},\mathcal{A},p,\gamma\right)$ where $\mathcal{X}$ is a finite state space, $\mathcal{A}$ the finite action space,  $p:\mathcal{X}\times\mathcal{A}\rightarrow\probspace(\mathcal{X})$ the transition kernel  and $\gamma\in [0,1)$ the discount factor. Let $\pi:\mathcal{X}\rightarrow\probspace(\mathcal{A})$ be a fixed policy. For convenience, let $P^\pi:\mathcal{X}\rightarrow\probspace(\mathcal{X})$ be the state transition kernel induced by the policy $\pi$.  We focus on reward-free MDPs instead of regular MDPs because we do not need reward functions for the rest of the discussion. 

Throughout, we assume tabular state representation where each state $x\in\mathcal{X}$ is equivalently encoded as a one-hot vector $x\in \mathbb{R}^\mathcal{X}$. This representation will be critical in establishing results that follow. In general, a representation matrix $\Phi\in\mathbb{R}^{|\mathcal{X}| \times k}$ embeds each state $x\in\mathcal{X}$ as a $k$-dimensional real vector $\Phi^T x\in\mathbb{R}^k$. In practice, we tend to have $k\ll|\mathcal{X}|$ where $|\mathcal{X}|$ is the cardinal of $\mathcal{X}$. The representation is generally shaped by learning signals such as TD-learning or auxiliary objectives. Good representations should entail sharing information between states, and facilitate downstream tasks such as policy evaluation or control. 

\begin{figure}[t]
    \centering
    \includegraphics[keepaspectratio,width=.45\textwidth]{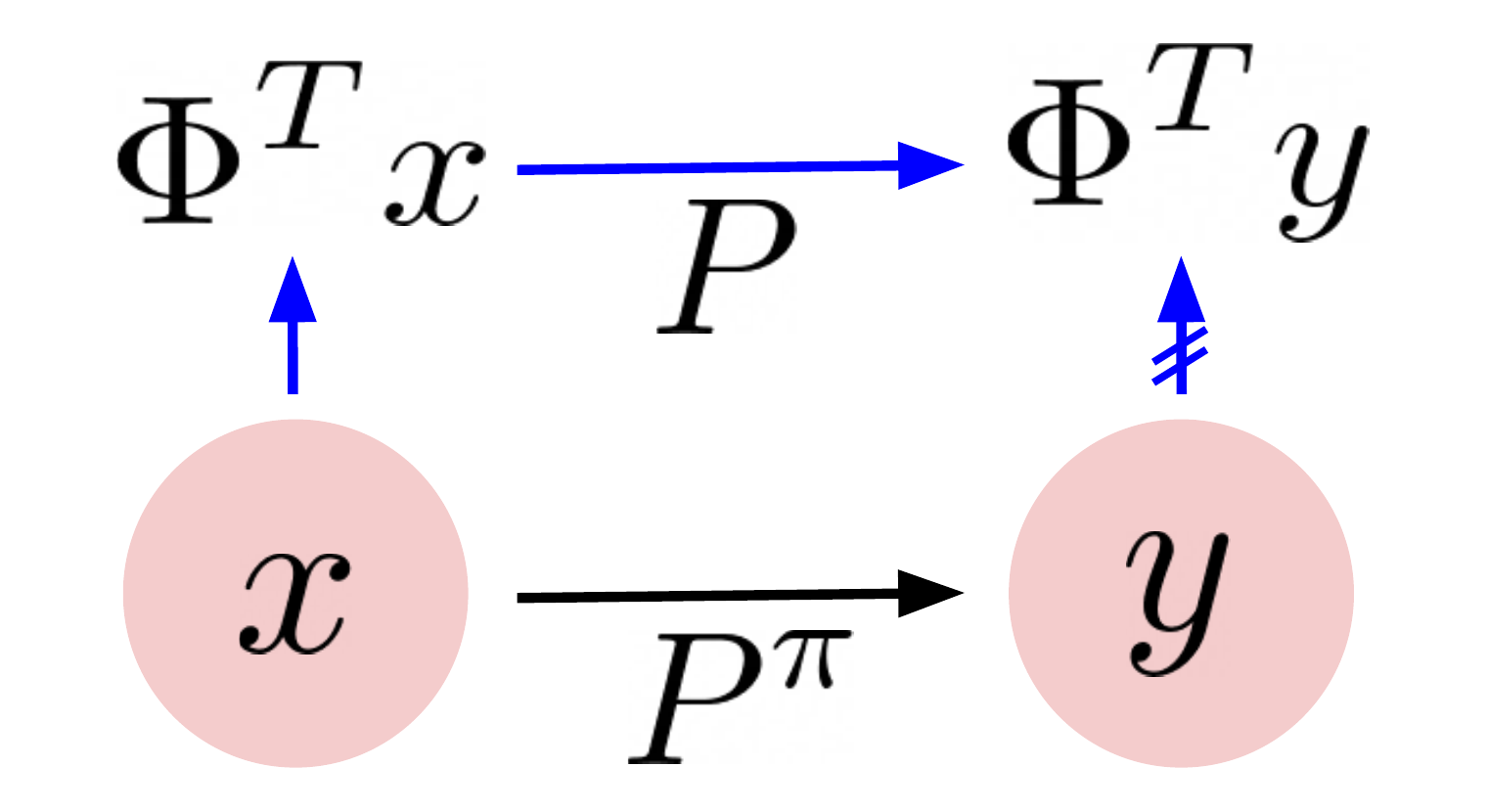}
    \caption{A diagram that outlines the conceptual components of self-predictive learning. The black arrow indicates sampling the transition $y\sim P^\pi(\cdot|x)$. the blue arrow indicates algorithmic components of self-predictive learning: predicting the next state representation $\Phi^Ty$ from the first state representation $\Phi^Tx$ using prediction matrix $P$. In practice as shown in \cref{eq:phi-update}, self-predictive learning stops  the gradient on the prediction target.} 
    \label{fig:self-predictive}
\end{figure}

\subsection{Self-predictive learning}

We introduce a mathematical framework for analyzing self-predictive learning, which seeks to capture the high level properties of its various algorithmic instantiations~\citep{schwarzer2021dataefficient,guo2020bootstrap,guo2022byol}. Throughout, assume we have access to state tuples $x,y\in\mathcal{X}$ sampled sequentially as follows,
\begin{align*}
    x\sim d, y \sim P^\pi(\cdot|x),
\end{align*}
where we use $x\sim d$ to denote sampling the state from a distribution defined by the probability vector $d\in\mathbb{R}^{|\mathcal{X}|}$. 
To model the transition in the representation space $\Phi^Tx\rightarrow \Phi^Ty$, we define $P\in\mathbb{R}^{k\times k}$ as the latent prediction matrix. The predicted latent at the next time step from $\Phi^Tx$ is $P^T\Phi^Tx$. The  goal is to minimize the reconstruction loss in the latent space,
\begin{align}
  \min_{\Phi,P} L(\Phi,P) \coloneqq  &\mathbb{E}_{x\sim d, y \sim P^\pi(\cdot|x)} \left[\left\lVert P^T \Phi^Tx - \Phi^Ty \right\rVert_2^2\right] \label{eq:latent-reconstruction-loss}.
\end{align}
As alluded to earlier, naively optimizing \cref{eq:latent-reconstruction-loss} may lead to trivial solutions such as $\Phi^\ast=0$, which also produces the optimal objective $L(\Phi^\ast,P)=0$. We next discuss how specific optimization procedures entail learning meaningful representation $\Phi$ and prediction function $P$.

\section{Understanding learning dynamics of self-predictive learning}
\label{sec:dynamics}

Assume the algorithm proceeds in discrete iterations $t\geq 0$.
In practice, the update of $\Phi$ follows a semi-gradient update through $L(\Phi,P)$ by stopping the gradient via the prediction target $\Phi^Ty$ 
\begin{align}
    \Phi_{t+1}\leftarrow\Phi_t -\eta \nabla_{\Phi_t} \mathbb{E} \left[\left\lVert P_t^T \Phi_t^Tx - \sg\left(\Phi_t^Ty\right) \right\rVert_2^2\right],\label{eq:phi-update}
\end{align}
where $\sg$ stands for stop-gradient and $\eta>0$ is a fixed learning rate. In the expectation, $x\sim d, y \sim P^\pi(\cdot|x)$ unless otherwise stated. 

\subsection{Non-collapse property of self-predictive learning}

\begin{figure}[t]
    \centering
    \includegraphics[keepaspectratio,width=.45\textwidth]{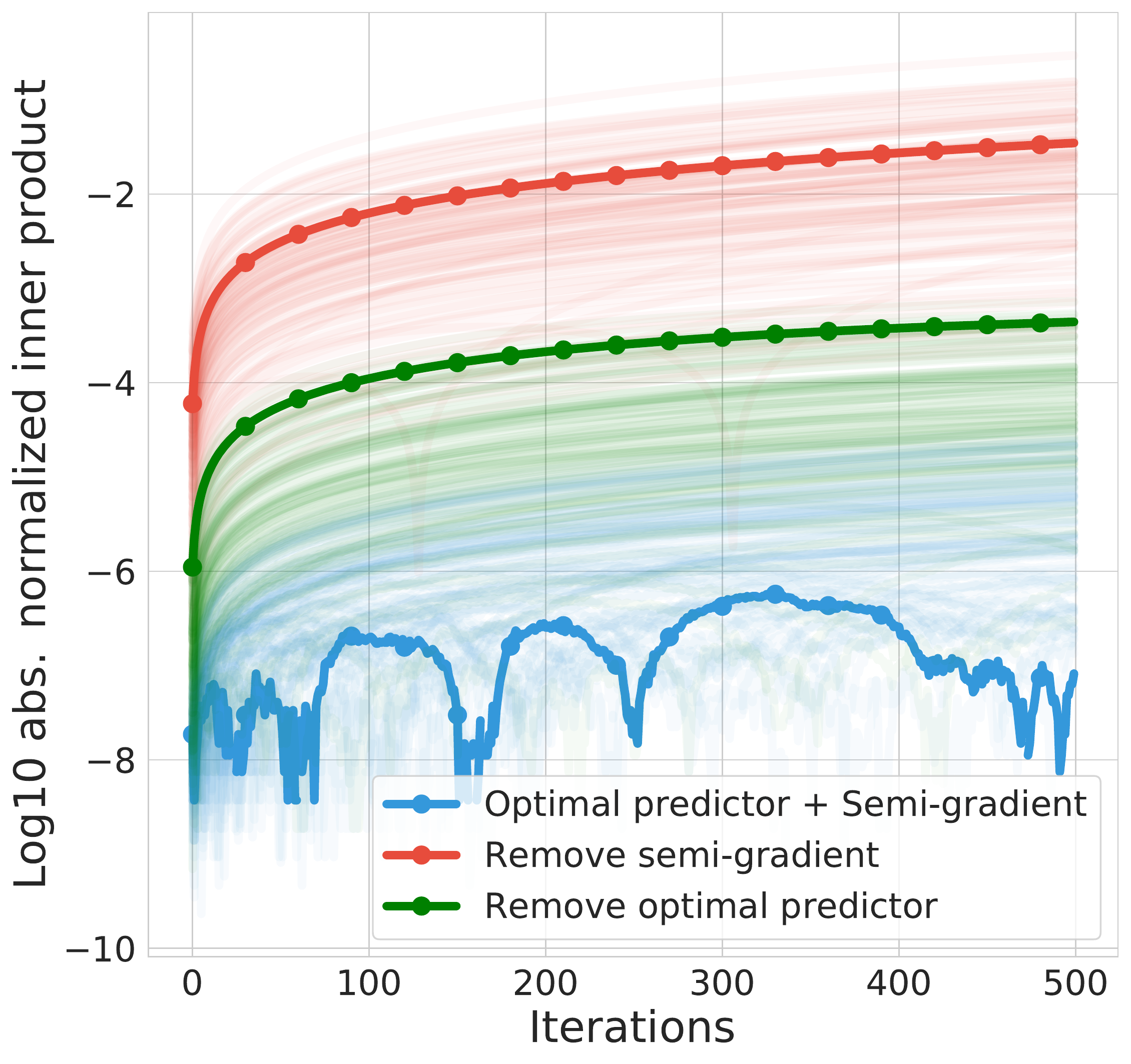}
    \caption{
    Absolute value of the inner product between the two (normalized) columns of $\Phi$, versus the number of iterations, for different variants of \cref{algo:byoldynamics}.
    Each light curve corresponds to one of $100$ independent runs over randomly generated MDPs, and the solid curve shows the median over runs. The experiments are based on the discretized dynamics (\cref{eq:phi-update}) with a small but finite learning rate.
    }
    \label{fig:collapse}
\end{figure}

We identified a key condition to ensure that the solution does not collapse to trivial solutions, that $P_t$ be optimized at a faster pace than $\Phi_t$. Intriguingly, this condition is compatible with a number of empirical observations made in prior work (e.g., Appendix I of \citealp{grill2020bootstrap} and \citealp{chen2021exploring} have both identified the importance of near optimal predictors). More formally, the prediction function $P_t$ is computed as one optimal solution to the loss function, fixing the representation $\Phi_t$,
\begin{align}
    P_t\in\arg\min_P \mathbb{E} \left[\left\lVert P^T \Phi_t^Tx - \Phi_t^Ty \right\rVert_2^2\right].\label{eq:P-update}
\end{align}
In practice, the above assumption might be approximately satisfied by the fact that the prediction function $P_t$ is often parameterized by a smaller neural network (e.g., an MLP or LSTM) compared to the representation $\Phi_t$ \citep[e.g., a deep ResNet,][]{schwarzer2021dataefficient,guo2020bootstrap,guo2022byol}, and hence can be optimized at a faster pace even with gradient descent. The pseudocode for such a self-predictive learning algorithm is in \cref{algo:byoldynamics}.

To understand the behavior of the joint updates in \cref{eq:phi-update,eq:P-update}, we propose to consider the behavior of the corresponding continuous time system. Let $t\geq 0$ be the continuous time index, the ordinary differential equation (ODE) systems jointly for $(\Phi_t,P_t)$ is
\begin{align}
\begin{split} \label{eq:ode}
   P_t &\in \arg\min_P L(\Phi_t,P),  \\ \dot{\Phi}_t &= - \nabla_{\Phi_t} \mathbb{E} \left[\left\lVert P_t^T \Phi_t^Tx - \sg\left(\Phi_t^Ty\right) \right\rVert_2^2\right].
\end{split}
\end{align}

Our theoretical analysis consists in understanding the behavior of the above ODE system.
The key result shows that the learning dynamics in \cref{eq:ode} does not lead to collapsed solutions.
\begin{restatable}{theorem}{theoremnocollapse}\label{theorem:nocollapse} Under the dynamics in \cref{eq:ode}, the covariance matrix $\Phi_t^T\Phi_t \in\mathbb{R}^{k\times k}$ is constant over time.
\end{restatable}

\begin{algorithm}[t]
\begin{algorithmic}
\STATE Representation matrix $\Phi_0\in\mathbb{R}^{|\mathcal{X}| \times k}$ for $k\leq |\mathcal{X}|$
\FOR{$t=1,2...T$}
\STATE Compute prediction matrix $P_t$ based on \cref{eq:P-update}.
\STATE Update representation $\Phi_t$ based on \cref{eq:phi-update}.
\ENDFOR
\STATE  Output final representation $\Phi_T$.
\caption{Self-predictive learning.\label{algo:byoldynamics}}
\end{algorithmic}
\end{algorithm}

The representation matrix decomposes into $k$ representation vectors $\Phi_t=[\phi_{1,t}\,\cdots\,\phi_{k,t}]$, where $\phi_{i,t} \in \mathbb{R}^\mathcal{X}$ for each $i=1,\ldots,k$.
Geometrically, we can visualize $(\phi_{i,t})_{i=1}^k$ as forming a basis of a $k$-dimensional subspace of $\mathbb{R}^{\mathcal{X}}$.
Throughout the learning process, all basis vectors rotate in the same direction, keeping the relative angles between basis vectors and their lengths unchanged. As a direct implication of the \emph{rotation} dynamics, it is not possible for $(\phi_{i,t})_{i=1}^k$ to start as different vectors but then converge to the same vector.
That is, the representation cannot collapse.

\begin{restatable}{corollary}{corollarynocollapse}\label{corollary:nocollapse} Under the dynamics in \cref{eq:ode}, the representation vectors $(\phi_{i,t})_{i=1}^k$ cannot converge to the same vector if they are initialized differently.
\end{restatable}

With the non-collapse behavior established under the dynamics in \cref{eq:ode}, we ask in hindsight what elements of the algorithm entail such a property. In addition to the faster paced optimization of the prediction matrix $P_t$, the semi-gradient update to $\Phi_t$ is also indispensable. Our result above provides a first theoretical justification to the ``latent bootstrapping''
technique in the RL case, which has been effective in empirical studies  \citep{grill2020bootstrap,schwarzer2021dataefficient,guo2020bootstrap}. See \cref{sec:discussion} for more discussions on the relation between our analysis and prior work in the non-contrastive unsupervised learning algorithms.

We illustrate the importance of the optimality of $P_t$ and the semi-gradient update with an empirical study.
We learned representations with $k=2$ vectors on randomly generated MDPs using three baselines: (1) using semi-gradient updates on $\Phi_t$ and with optimal predictors $P_t$, which strictly adheres to \cref{algo:byoldynamics}; (2) using the optimal predictors $P_t$ but replacing the semi-gradient update by a full gradient on $\Phi_t$, i.e., allowing gradients to flow into $\Phi^T y$; (3) using the semi-gradient update but corrupting the optimal predictor with some zero-mean noise at each iteration.

\Cref{fig:collapse} shows the absolute value of the inner product between the two (normalized) columns of $\Phi$, also known as the cosine similarity, versus the number of iterations.

The matrix $\Phi$ is initialized to be orthogonal, so we expect the two columns in $\Phi$ to remain close to orthogonal throughout the learning dynamics (\cref{eq:phi-update}) with a small but finite learning rate $\eta$.
Therefore, the larger values the curves take in \cref{fig:collapse}, the stronger the evidence of collapse, and
indeed as we claimed, when either semi-gradient or optimal predictor are removed from the learning algorithm, the representation columns start to collapse. In practice, having an optimal predictor is a stringent requirement; we carry out more extensive ablation study in \cref{sec:exp}. See \cref{appendix:exp} for more details on the tabular experiments.

\paragraph{Non-collapse property for a general loss function.} To make our analysis simple, we focused on the squared loss function between the prediction $P^T\Phi^Tx$ and target $\Phi^Ty$. We note that the non-collapse property in \cref{theorem:eigen} holds more generally for losses of the form $L\left(\Phi,P\right)$.
Our result also applies to a slightly modified variant of the cosine similarity loss, which is more commonly used in practice \citep{grill2020bootstrap,chen2021exploring,schwarzer2021dataefficient,guo2022byol}. 
See \cref{appendix:loss} for more details.

The non-collapse property shows that the covariance matrix $\Phi_t^T\Phi_t$, which measures the level of diversity across representation vectors $(\phi_{i,t})_{i=1}^t$ is conserved.
In the next section, we will discuss the connection between the learning dynamics and spectral decomposition on the transition matrix $P^\pi$. To facilitate the discussion, we make an idealized assumption that the representation columns are initialized orthonormal. 

\begin{restatable}{assumption}{assumptioninit}\label{assumption:init} (\textbf{Orthonormal Initialization}) The representations are initialized orthonormal: $\Phi_0^T\Phi_0=I_{k\times k}$.
\end{restatable}

Note that the assumption is approximately valid e.g., when entries of $\Phi_0$ are sampled i.i.d.~from an isotropic distribution and properly scaled, and when the state space is large relative to the representation dimension $k\ll|\mathcal{X}|$ (see, e.g., \citep{gonzalez2018euclidean} as a related reference). This is the case if the representation capacity is smaller than the state space (e.g., small network vs. complex image observation).

\subsection{Self-predictive learning as eigenvector decomposition}

For simplicity, henceforth, we assume a uniform distribution over the first-state. This assumption is made implicitly in a number of prior work on TD-learning or representation learning for RL \citep{parr2008analysis,song2016linear,behzadian2019fast,lyle2021effect}.

\begin{restatable}{assumption}{assumptionD}\label{assumption:D} (\textbf{Uniform distribution}) The first-state distribution is uniform:  $d=|\mathcal{X}|^{-1}1_{|\mathcal{X}|}$.
\end{restatable}

For the rest of the paper, we always assume \cref{assumption:init} and  \cref{assumption:D} to hold. As a result, the learning dynamics in \cref{eq:ode} reduces to the following:
\begin{align}
    P_t = \Phi_t^T P^\pi\Phi_t,  \ \ \dot{\Phi}_t = \left(I-\Phi_t\Phi_t^T\right) P^\pi \Phi_t (P_t)^T \label{eq:uniformD-ode}
\end{align}
We provide detailed derivations of the ODE in \cref{appendix:ode}.

\paragraph{Remarks.} As a sanity check and a special case, assume the representation matrix is the identity $\Phi_t=I$, in which case $\Phi_t^Tx=x$ recovers the tabular representation. In this case we have $P_t=P^\pi$ and the latent prediction recovers the original state transition matrix. 

A useful property of the above dynamical system is the set of critical points where $\dot{\Phi}_t=0$.  Below we provide a characterization of such critical points.

\begin{restatable}{lemma}{lemmacritical}\label{lemmacritical} Assume $P^\pi$ is real diagonalizable and let  $(u_i)_{i=1}^{|\mathcal{X}|}$ be its set of $|\mathcal{X}|$ distinct eigenvectors. Let $\mathcal{C}_{P^\pi}$ be the set of critical points of \cref{eq:uniformD-ode}. 
Then $\mathcal{C}_{P^\pi}$ contains all matrices whose columns are orthonormal, and have the same span as a set of $k$ eigenvectors.
\end{restatable}

\paragraph{Remarks on the critical points.}
\cref{lemmacritical} implies that when $P^\pi$ only has eigenvectors, any matrix consisting of a subset of $k$ eigenvectors $(u_{i_j})_{j=1}^k$ (as well as any set of $k$ orthonormal columns with the same span) is a critical point to the self-predictive dynamics in \cref{eq:uniformD-ode}. However, this does not mean that $\mathcal{C}_{P^\pi}$ only consists of such critical points. As a simple example, consider the transition matrix
\begin{align} \label{eq:example1}
   P^\pi = 
\begin{bmatrix}
    0.1 & 0.9 \\
    0.9 & 0.1 \\
\end{bmatrix}
\end{align}
and when $k=1$, in which case $\Phi_t$ is a $2$-d vector. In addition to the two eigenvectors of $P^\pi$, there are at least four other non-eigenvector critical points (shown in \cref{fig:criticalpoints}). See  \cref{appendix:critical} for more detailed derivations. We leave a more comprehensive study of such critical points in the general case to future work. When $P^\pi$ has complex eigenvectors, the structure of the critical points to \cref{eq:uniformD-ode} also becomes more complicated. 

Importantly, under the assumption in \cref{lemmacritical}, not all critical points are equally informative. Arguably, the top $k$ eigenvectors of $P^\pi$ with the largest absolute valued eigenvalues, should contain the most information about the matrix because they reflect the high variance directions in the one-step transition. This is the motivation behind compression algorithms such as PCA. We now show that when $P^\pi$ is symmetric, intriguingly, the learning dynamics maximizes a trace objective that measures the variance information contained in $P^\pi$, and that this objective is maximized by the top $k$ eigenvectors.

\begin{restatable}{theorem}{theoremeigen}\label{theorem:eigen} If $P^\pi$ is symmetric, then under \cref{assumption:init} and learning dynamics  \cref{eq:uniformD-ode}, the trace objective is non-decreasing $\dot{f}\geq 0$, where
\begin{align*}
    f\left(\Phi_t\right) \coloneqq \text{Trace}\left(\left(\Phi_t^T P^\pi \Phi_t\right)^T\left(\Phi_t^T P^\pi \Phi_t\right)\right).
\end{align*}
If $\Phi_t\notin\mathcal{C}_{P^\pi}$, then $\dot{f}>0$.
Under the constraint $\Phi^T\Phi=I$, the maximizer to $f(\Phi)$ is any set of $k$ orthonormal vectors which span the principal subspace, i.e., with the same span as the $k$ eigenvectors of $P^\pi$ with top absolute eigenvalues. 
\end{restatable}

To see that the trace objective $f$ measures useful information contained in $P^\pi$, for now let us constrain the arguments to $f$ to be the set of $k$ eigenvectors $(u_{i_j})_{j=1}^k$ of $P^\pi$. In this case, $f\left([u_{i_1}...u_{i_k}]\right)=\sum_{j=1}^k \lambda_{i_j}^2$ is the sum of the corresponding squared eigenvalues. This implies $f$ is a useful measure on the spectral information contained in $P^\pi$. 

\cref{theorem:eigen} also shows that as long as $\Phi_t$ is not at a stationary point contained in $\mathcal{C}_{P^\pi}$, the trace objective $f(\Phi_t)$ makes strict improvement over time. Equivalently, this means $\Phi_t$ has the tendency to move towards representations with high trace objective, e.g., subspaces of eigenvectors with high trace objective. In other words, we can understand the dynamics of $\Phi_t$ as principal subspace PCA on the transition matrix.

\begin{figure}[t]
    \centering
    \includegraphics[keepaspectratio,width=.45\textwidth]{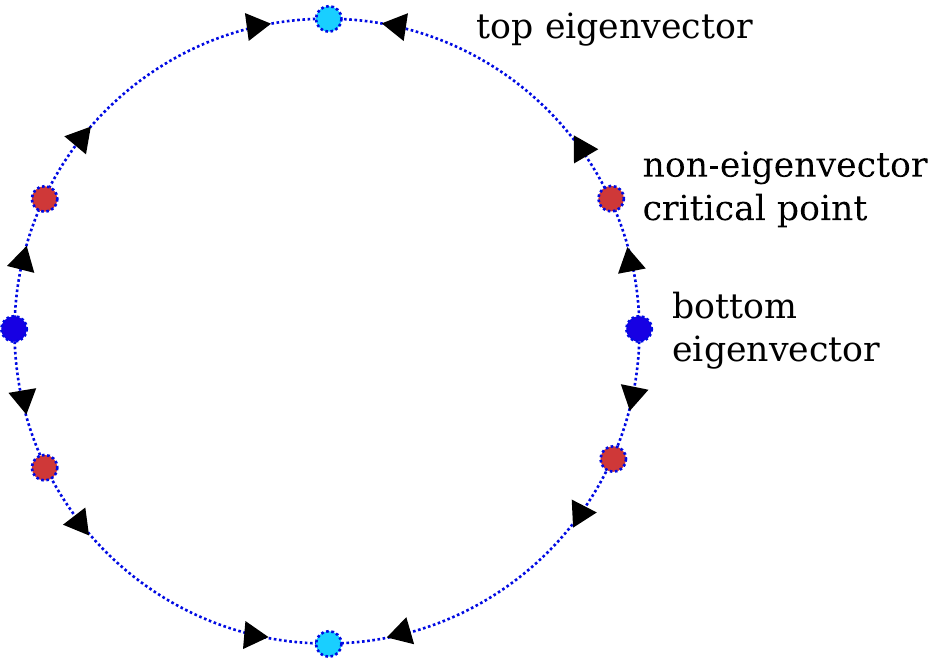}
    \caption{Critical points and local dynamics of the example MDP in \cref{eq:example1}. We consider $k=1$ so representations $\Phi_t$ are $2$-d unit vectors. There are four eigenvector critical points (light and dark blue) and four non-eigenvector critical points (red) of the ODE, shown on the unit circle. The black arrows show the local update direction based on the ODE. Initialized near the bottom eigenvector, the dynamics converges to one of the four non-eigenvector critical points and not to the top eigenvector. See \cref{appendix:critical} for more detailed explanations.
   } 
    \label{fig:criticalpoints}
\end{figure}

\paragraph{Remarks on the convergence.}
Thus far, there is no guarantee that $\Phi_t$ converges to the top $k$ eigenvectors as in general there is a chance that the dynamics converges other critical points. We revisit the simple example in \cref{eq:example1}, where in \cref{fig:criticalpoints} we mark all critical points on the unit circle (four eigenvector and four non-eigenvector critical points). When initialized near the bottom eigenvector, the dynamics converges to one of the non-eigenvector critical points instead of the top eigenvector. 

Nevertheless, the local improvement property of the learning dynamics can be very valuable in large-scale environments. We leave a more refined study on the convergence properties of self-predictive learning dynamics to future work.

\paragraph{Remarks on the case with non-symmetric $P^\pi$.} \cref{theorem:eigen} is obtained under the idealized assumption that $P^\pi$ is symmetric. In general, when $P^\pi$ is non-symmetric, the improvement in the trace objective $f(\Phi_t)$ is not necessarily monotonic.
In fact, it is possible to find instances of $\Phi_t$ where the dynamics  decreases the trace objective $f$. 
A plausible  explanation is that since the dynamics of $\Phi_t$ can be understood as gradient-ascent based PCA on $P^\pi$, the PCA objective is only well defined when the data matrix $P^\pi$ is symmetric. Motivated by the limitation of the self-predictive learning and its connection to PCA, we propose a novel self-predictive algorithm with two representations. We will introduce such a method in \cref{sec:double} and reveal how it generalizes self-predictive learning to carrying out SVD instead of PCA on $P^\pi$.

Before moving on, we empirically assess how much impact that the level of symmetry of $P^\pi$ has on the trace maximization property. We carried out simulations on $100$ randomly generated tabular MDPs, by unrolling the exact ODE dynamics in \cref{eq:ode} and measured the evolution of the trace objective $f_t = \text{Trace}\left((\Phi_t^TP^\pi\Phi_t)^T \Phi_t^TP^\pi\Phi_t\right)$.
\Cref{fig:violation} shows the ratio between the trace objective $f_t$ and the the value of the objective for the top $k$ eigenvectors of $P^\pi$, versus the number of training iterations $t$.

\begin{figure}[t]
    \centering
    \includegraphics[keepaspectratio,width=.45\textwidth]{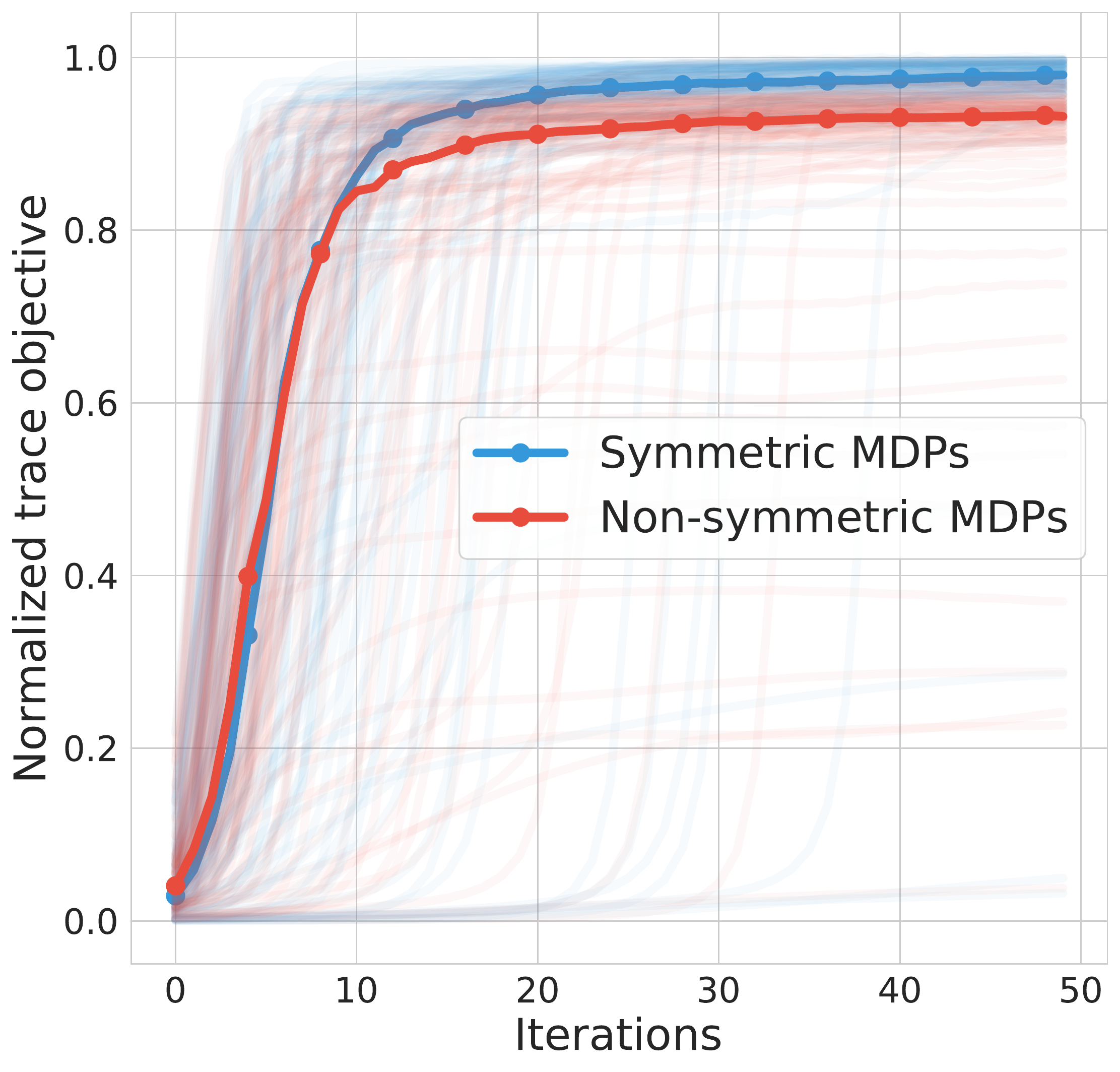}
    \caption{
    Ratio between the trace objective $f_t=\text{Trace}\left((\Phi_t^TP^\pi\Phi_t)^T \Phi_t^TP^\pi\Phi_t\right)$ and the value of the objective for the top $k$ eigenvectors of $P^\pi$, versus the number of training iterations.
    Each light curve corresponds to one of $100$ independent runs over randomly generated MDPs, and the solid curve shows the median over runs. The experiments are based on the exact ODE dynamics in \cref{eq:uniformD-ode}.} 
    \label{fig:violation}
\end{figure}

\cref{fig:violation} shows that when $P^\pi$ is symmetric, the trace objective smoothly improves over time, as predicted by theory. When $P^\pi$ is non-symmetric, the improvement in trace objective is not guaranteed to be monotonic and, indeed, this is the case for some runs.
However in our experiments, over time, the objective improved by a large margin compared to initialization, though not necessarily converging to the maximum possible values. The numerical evidence shows that the learning dynamics can still capture useful information about the transition dynamics for certain non-symmetric MDPs. It is, however, possible to design non-symmetric $P^\pi$ on which self-predictive dynamics barely increases the trace objective (see \cref{sec:double}).
\Cref{appendix:exp} contains additional results and more details about the experiment details.

\section{Extensions of the theoretical analysis}
\label{sec:extensions}

We have chosen to analyze the learning dynamics under arguably the simplest possible model setup. This helps elucidate important features about the self-predictive learning dynamics, but also leaves room for extensions. We discuss a few possibilities.

\paragraph{Additional prediction function.}
Practical algorithms such as SPR \citep{schwarzer2021dataefficient} and BYOL-RL (the representation learning component of BYOL-Explore \citep{guo2022byol}) usually employ an additional prediction function, on top of the prediction function $P$ which models the latent transition dynamics. In our framework, this can be modeled as an additional prediction matrix $Q\in\mathbb{R}^{k\times k}$ with the overall loss function as follows
\begin{align*}
   \mathbb{E}\left[ \left\lVert Q^T P^T \Phi^T x - \Phi^T y \right\rVert_2^2\right].
\end{align*}
The roles of $P$ and $Q$ are different. $P$ is meant to model the latent transition dynamics, while $Q$ provides extra degrees of freedom to match the predicted latent $Q^T P^T \Phi^T x$ to the next state representation $\Phi^T y$. Though such a combination of $Q,P$ seems redundant at the first sight, the extra flexibility entailed by the additional prediction proves very important in practice ($Q$ is usually implemented as a MLP on top of the output of $P$, which is implemented as a LSTM \citep{schwarzer2021dataefficient,guo2020bootstrap}). Our theoretical result can be extended to this case by treating the composed matrix $PQ$ as a whole during optimization, to ensure the non-collapse of $\Phi$.

\paragraph{Multi-step and action-conditional latent prediction.}

In practice, making multi-step predictions significantly improves the performance \citep{schwarzer2021dataefficient,guo2020bootstrap,guo2022byol}. In our framework, this can be understood as the loss function
\begin{align*}
   \mathbb{E}_{x_n\sim \left(P^\pi\right)^n (\cdot|x_0)}\left[ \left\lVert P^T \Phi^T x_0 - \Phi^T x_n \right\rVert_2^2\right].
\end{align*}
In this case, our result in \cref{sec:dynamics} suggests that the self-prediction carries out spectral decomposition on the $n$-step transition $\left(P^\pi\right)^n$. Another important practical component is that latent transition models are usually action-conditional. In the one-step case, this can be understood as parameterizing multiple prediction matrices and representation matrices $(P_a,\Phi_a)_{a\in\mathcal{A}}$ and . The loss function naturally becomes
\begin{align*}
   \mathbb{E}_{x_n\sim P^\pi\left (\cdot|x_0,a\right),a\sim \pi(\cdot|x_0)}\left[ \left\lVert P_a^T \Phi_a^T x_0 - \Phi_a^T x_n \right\rVert_2^2\right].
\end{align*}
The latent prediction matrix $P_a$ and representation $\Phi_a$ effectively carry out spectral decomposition on the Markov matrix $P^{\pi_a}$, which is the transition matrix of policy $\pi_a$ that takes action $a$ in all states.

\paragraph{Partial observability.} In many practical applications, the environment is better modeled as a partially observable MDP (POMDP; \citep{cassandra1994acting}). As a simplified setup, consider at time $t$ the agent has access to the current history $h_t=\left(o_s\right)_{s\leq t}\in\mathcal{H}$ which consists of observations $o_s\in\mathcal{O}$ in past time steps. Fixing the agent's policy $\pi:\mathcal{H}\rightarrow\mathcal{P}(\mathcal{A})$, let $\tilde{P}^\pi(h'|h)$ denote the distribution over the next observed history $h'$ given $h$. Drawing direct analogy to the MDP case, one possible loss function is
\begin{align*}
    \mathbb{E}_{h'\sim \tilde{P}^\pi(\cdot|h)}\left[\left\lVert P^T\Phi^Th-\Phi^Th'\right\rVert_2^2\right].
\end{align*}
In practice, $\Phi\in\mathbb{R}^{|\mathcal{H}|\times k}$ is often implemented as a recurrent function such as LSTM (see, e.g., BYOL-RL \citep{guo2020bootstrap} as one possible implementation and find its details in \cref{appendix:byolrl}), to avoid the explosion in the size of the set of all histories $|\mathcal{H}|$. Under certain conditions, our analysis can be extended to spectral decomposition on the history transition matrix $P^\pi(h'|h)$. However, given recent empirical advances achieved by self-predictive learning algorithm in partially observable environments \citep{guo2022byol}, potentially a more refined analysis is valuable in better bridging the theory-practice gap in the POMDP case.

\paragraph{Finite learning rate and other factors.} Our analysis and result heavily rely on the assumption of continuous time dynamics. In practice, updates are carried out on discrete time steps with a finite learning rate. Through experiments, we observe that representations tend to partially collapse when learning rates are finite, though they still manage to capture spectral information about the transition matrix. We also study the impact of other factors such as non-optimal prediction matrix and delayed target network, see \cref{appendix:exp} for such ablation study. Formalizing such results in theory would be an interesting future direction. 

\section{Bidirectional self-predictive learning with left and right representations}
\label{sec:double}

Thus far, we have established a few important properties of the self-predictive learning dynamics.
However, we have alluded to the fact that self-predictive learning dynamics can ill-behave in certain cases. 

With insights derived from previous sections,
we now introduce a novel self-predictive learning algorithm that makes use of two representations $\Phi_t, \tilde{\Phi}_t\in\mathbb{R}^{|\mathcal{X}| \times k}$ and two latent prediction matrices $P,\tilde{P}\in\mathbb{R}^{k\times k}$. 
We refer to $\Phi_t$ as the \emph{left representation} and $\tilde{\Phi}_t$ the \emph{right representation}, for reasons that will be clear shortly. With $\Phi_t$, we make forward prediction through $P$, using prediction target computed from $\tilde{\Phi}_t$;  with $\tilde{\Phi}_t$, we make backward prediction through $\tilde{P}$, using prediction target computed from $\Phi_t$. Both representations follow semi-gradient updates:
\begin{align}
\label{eq:double-update}
\begin{split}
    \Phi_{t+1}&\leftarrow\Phi_t -\eta \nabla_{\Phi_t} \mathbb{E} \left[\left\lVert P_t^T \Phi_t^Tx - \sg\left(\tilde{\Phi}_t^Ty\right) \right\rVert_2^2\right],\\
    \tilde{\Phi}_{t+1}&\leftarrow\tilde{\Phi}_t -\eta \nabla_{\tilde{\Phi}_t} \mathbb{E} \left[\left\lVert \tilde{P}_t^T \tilde{\Phi}_t^Ty - \sg\left(\Phi_t^Tx\right) \right\rVert_2^2\right].
\end{split}
\end{align}
Similar to the analysis before, we assume $P_t,\tilde{P}_t$ are optimally adapted to the representations, by exactly solving the forward and backward least square prediction problems. This is a key requirement to ensure non-collapse in the new learning dynamics (similar to the self-predictive dynamics in \cref{eq:uniformD-ode}). The pseudocode for the bidirectional self-predictive learning algorithm is in \cref{algo:doublebyoldynamics}.

\begin{algorithm}[t]
\begin{algorithmic}
\STATE Representation matrix $\Phi_0\tilde{\Phi}_0\in\mathbb{R}^{|\mathcal{X}| \times k}$ for $k\leq |\mathcal{X}|$
\FOR{$t=1,2...T$}
\STATE Compute optimal forward and backward prediction matrix $(P_t,\tilde{P}_t)$ based on \cref{eq:double-ode}.
\STATE Update representations $(\Phi_t,\tilde{\Phi}_t)$ based on \cref{eq:double-update}.
\ENDFOR
\STATE  Output final representation $(\Phi_T,\tilde{\Phi}_T)$.
\caption{Bidirectional self-predictive learning.\label{algo:doublebyoldynamics}}
\end{algorithmic}
\end{algorithm}

For notational simplicity, we denote the forward and backward prediction losses as $L_\mathrm{f}(\Phi,P)$ and $L_\mathrm{b}(\tilde{\Phi},\tilde{P})$ respectively.
The continuous time ODE system for the joint variable $(P_t,\tilde{P}_t,\Phi_t,\tilde{\Phi}_t)$ is
\begin{align}
\label{eq:double-ode}
\begin{split}
    P_t &\in \arg\min_P L_\mathrm{f}(\Phi_t,P),\ \ \\ \dot{\Phi}_t &= - \nabla_{\Phi_t} \mathbb{E} \left[\left\lVert P^T \Phi^Tx - \sg\left(\tilde{\Phi}^Ty\right) \right\rVert_2^2\right], \\
    \tilde{P}_t &\in \arg\min_{\tilde{P}} L_\mathrm{b}(\tilde{\Phi}_t,\tilde{P}),\ \  \\ \dot{\tilde{\Phi}}_t &= - \nabla_{\tilde{\Phi}_t} \mathbb{E} \left[\left\lVert \tilde{P}^T \tilde{\Phi}^Ty - \sg\left(\Phi^Tx\right) \right\rVert_2^2\right].
\end{split}
\end{align}

Similar to \cref{theorem:nocollapse},
the non-collapse property for both the left and right representations follows.
\begin{restatable}{theorem}{theoremnocollapsedouble}\label{theorem:nocollapsedouble} Under the bidirectional self-predictive learning dynamics in \cref{eq:double-ode}, the covariance matrices $\Phi_t^T\Phi_t \in\mathbb{R}^{k\times k}$ and $\tilde{\Phi}_t^T\tilde{\Phi}_t \in\mathbb{R}^{k\times k}$ are both constant matrices over time.
\end{restatable}

As before, to simplify the presentation, we make the assumption that both left and right representations are initialized orthonormal (cf.~\cref{assumption:init}):
\begin{restatable}{assumption}{assumptioninitdouble}\label{assumption:init-double} (\textbf{Orthonormal Initialization}) The left and right representations are both initialized orthonormal $\Phi_0^T\Phi_0=\tilde{\Phi}_0^T\tilde{\Phi}_0=I_{k\times k}$.
\end{restatable}

\subsection{Bidirectional self-predictive learning as singular value decomposition}

In self-predictive learning, the forward prediction derives from the fact that the forward process $x\rightarrow y$ follows from a Markov chain. Following a similar argument, for the backward prediction to be sensible, we need to ensure that the reverse process $y\rightarrow x$ is also a Markov chain. Technically, this means we require $\left(P^\pi\right)^T$ to be a transition matrix too, which models the backward transition process. Importantly, this is a much weaker assumption than $P^\pi$ be symmetric, as required by the self-predictive learning dynamics (\cref{theorem:eigen}).

\begin{restatable}{assumption}{assumptiondoublestochastic}\label{assumption:doublestochastic}  $P^\pi$ is a doubly stochastic matrix, i.e., $\left(P^\pi\right)^T$ is also a transition matrix.
\end{restatable}

Under assumptions above, the learning dynamics in \cref{eq:double-ode} reduces to the following set of ODEs:
\begin{align}
    P_t &= \Phi_t^T P^\pi\tilde{\Phi}_t,  \ \ \dot{\Phi}_t = \left(I-\Phi_t\Phi_t^T\right) P^\pi \tilde{\Phi}_t (P_t)^T \nonumber \\
    \tilde{P}_t &= \tilde{\Phi}_t^T \left(P^\pi\right)^T\Phi_t,  \ \ \dot{\tilde{\Phi}}_t = \left(I-\tilde{\Phi}_t\tilde{\Phi}_t^T\right) \left(P^\pi\right)^T \Phi_t (\tilde{P}_t)^T 
    \label{eq:double-ode-2}
\end{align}

Let $P^\pi=U\Sigma V^T$ be the singular value decomposition (SVD) of $P^\pi$, where $\Sigma=\text{diag}(\sigma_1,\sigma_2...\sigma_{|\mathcal{X}|})$ is a diagonal matrix with non-negative diagonal entries. We call any $i$-th column of $U$ and $V$, denoted as $(u_i,v_i)$, a singular vector pair.
As before, we start by examining the critical points of the bidirectional self-predictive dynamics.
\begin{restatable}{lemma}{lemmacriticaldouble}\label{lemmacriticaldouble}  Let $\mathcal{\tilde{C}}_{P^\pi} \subset \mathbb{R}^{|\mathcal{X}| \times k} \times \mathbb{R}^{|\mathcal{X}| \times k}$ be the set of critical points to \cref{eq:double-ode-2}. Then $\mathcal{\tilde{C}}_{P^\pi}$ contains any pair of matrices, whose columns are orthonormal and have the same span as $k$ of singular vector pairs.
\end{restatable}

\cref{lemmacriticaldouble} implies that any $k$ pairs of SVD vectors are a critical point to the learning dynamics. However, similar to the self-predictive learning dynamics, not all critical points are equally informative. We propose a SVD trace objective $\tilde{f}(\Phi_t,\tilde{\Phi}_t)$, which measures the information contained in $k$ singular vector pairs. Interestingly, the bidirectional self-predictive learning dynamics locally improves such an objective.
\begin{restatable}{theorem}{theoremsvd}\label{theorem:svd} Under \cref{assumption:init-double} and the learning dynamics in \cref{eq:double-ode-2}, the following SVD trace objective is non-decreasing $\dot{\tilde{f}}\geq 0$, where
\begin{align*}
    \tilde{f}\left(\Phi_t,\tilde{\Phi}_t\right) \coloneqq \text{Trace}\left(\left(\Phi_t^T P^\pi \tilde{\Phi}_t\right)^T\left(\Phi_t^T P^\pi \tilde{\Phi}_t\right)\right).
\end{align*}
If $(\Phi_t,\tilde{\Phi}_t)\notin\mathcal{\tilde{C}}_{P^\pi}$, then $\dot{\tilde{f}}>0$.
Under the constraint $\Phi^T\Phi=\tilde{\Phi}^T\tilde{\Phi}=I$, the maximizer to $\tilde{f}(\Phi,\tilde{\Phi})$ is any two sets of $k$ orthonormal vectors with the same span as the $k$ singular vector pairs of $P^\pi$ with top singular values.
\end{restatable}

To verify that the SVD trace objective $\tilde{f}$ provides an information measure on the representation vectors $(\Phi_t,\tilde{\Phi}_t)$, we constrain arguments of $\tilde{f}$ to be the set of $k$ singular vector pairs $(u_{i_j},v_{i_j})_{j=1}^k$ of $P^\pi$, then $\tilde{f}\left([u_{i_1}...u_{i_k}],[v_{i_1}...v_{i_k}]\right)=\sum_{j=1}^k \sigma_{i_j}^2$ is the sum of the corresponding squared singular values.  The top $k$ singular vector pairs maximize this objective, and hence contain the most information about $P^\pi$ based on this measure.

\cref{theorem:svd} shows that as long as either one of the two representations are not at the critical points, i.e., $\dot{\Phi}_t\neq 0$ or $\dot{\tilde{\Phi}}_t\neq 0$, the SVD trace objective $\tilde{f}(\Phi_t,\tilde{\Phi}_t)$ is being strictly improved under the bidirectional self-predictive learning dynamics. 
Equivalently, this implies the left and right representations $(\Phi_t,\tilde{\Phi}_t)$ tend to move towards singular vector pairs with high SVD trace objective, i.e., seeking more information about the transition dynamics.

\paragraph{Two representations vs. one representation.} Looking beyond the ODE analysis, we explain why having two separate representations are inherently important for representation learning in general. For a transition matrix $P^\pi$, its left and right singular vectors in general differ.
bidirectional self-predictive learning provides the flexibility to learn both left and right singular vectors in parallel, without having to compromise their differences. On the other hand, a single representation will need to interpolate between left and right singular vectors, which may lead to  non-monotonic behavior in the trace objective as alluded to earlier.

Consider a very simple transition matrix with $|\mathcal{X}|=3$ states that illustrate the failure mode of the self-predictive learning dynamics with a single representation,
\begin{align*}
   P^\pi = 
\begin{bmatrix}
    0 & 1/2 & 1/2 \\
    0 & 1/2 & 1/2 \\
    1 & 0 & 0 \\
\end{bmatrix}.
\end{align*}
By construction, its top left and right singular vectors differ greatly. 
We simulated the self-predictive learning dynamics (with a single representation, \cref{eq:ode}) and the bidirectional self-predictive learning dynamics (\cref{eq:double-ode}) in an MDP with this transition matrix, and measured the evolution of the two trace objectives, $f_t$ and $\tilde{f_t}$. \Cref{fig:double_single} shows the ratio between the trace objectives and maximum value of $\tilde{f}$ obtained at the top $k$ singular vector pairs, versus the number of training iterations $t$.
\begin{figure}[t]
    \centering
    \includegraphics[keepaspectratio,width=.45\textwidth]{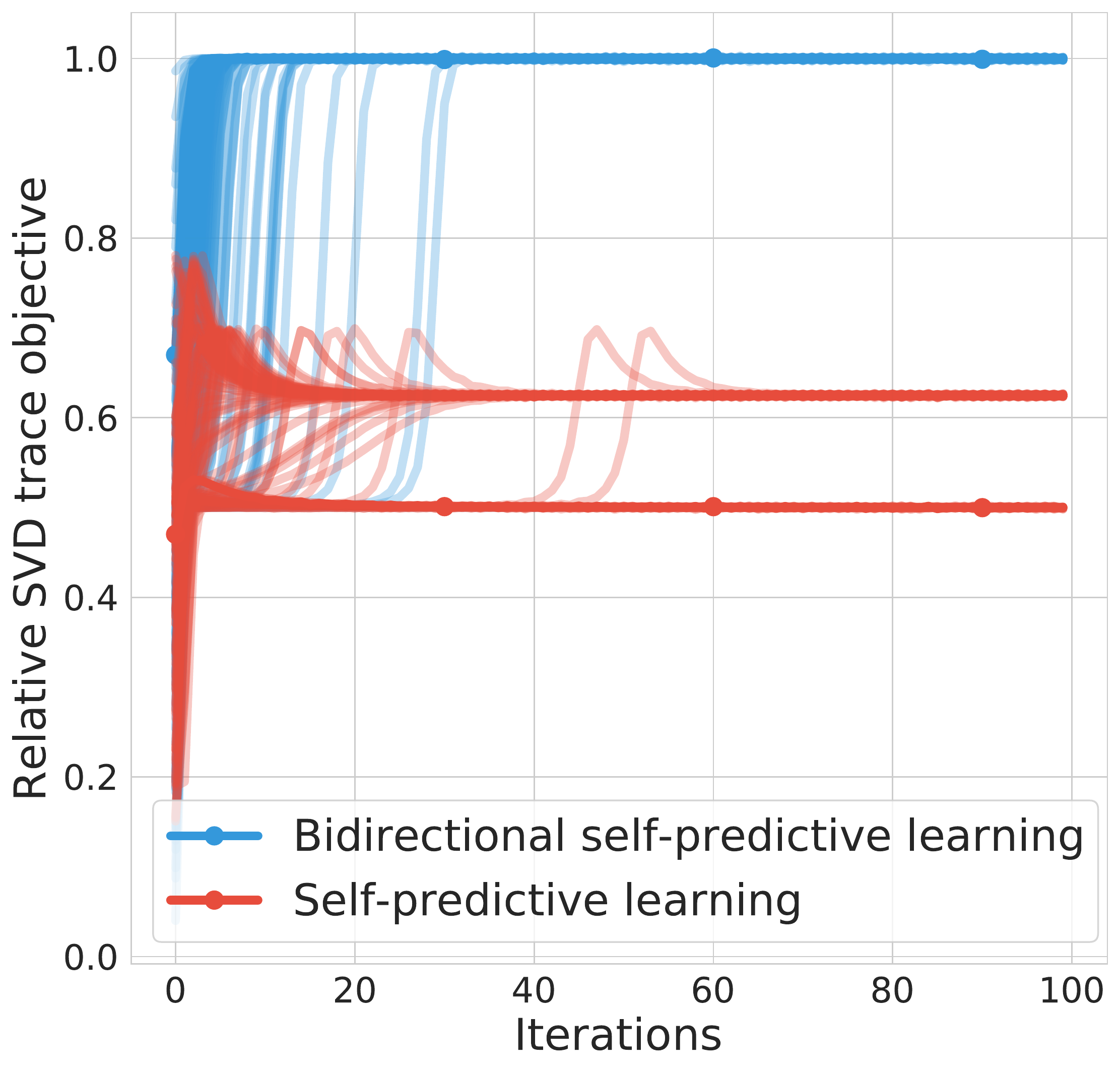}
    \caption{
    Ratio between the trace objectives ($f_t$ in red for self-predictive learning, following \cref{eq:ode}; and $\tilde{f}_t$ in blue for bidirectional self-predictive learning, following \cref{eq:double-ode}) and the value of $\tilde{f}$ for the top $k$ singular vector pairs of $P^\pi$, versus the number of training iterations.
    Each light curve corresponds to one of $100$ independent runs over random orthonormal initializations of $\Phi$ on the same MDP designed so that the left and right singular vectors of $P^\pi$ are very different. 
    The solid curve shows the median over runs.} 
    \label{fig:double_single}
\end{figure}
The bidirectional self-predictive learning improves the objective steadily over time. In contrast, the single representation dynamics mostly halt at the initialized value, due to the limited capacity of one representation to combine two highly distinct singular vectors. See more details in \cref{appendix:double-single}. 

In addition to the improved stability shown in the example above, bidirectional self-predictive learning also captures more comprehensive information about the transition dynamics, compared to a single representation. While it is known that left singular vectors of $P^\pi$ can approximate value functions $V^\pi$ with provably low errors for general transition matrix \citep{behzadian2019fast}, it is challenging to learn the left singular vectors as standalone objects. Bidirectional self-predictive learning captures both left and right representations at the same time, entailing a better approximation to both left and right top singular vectors. In large-scale experiments, since bidirectional self-predictive learning consists of both forward and backward predictions, it can provide richer signals for representation learning when combined with nonlinear function approximation (\cref{sec:exp}).

\section{Prior Work}
\label{sec:discussion}

\paragraph{Non-collapse mechanism of self-predictive learning dynamics.} Self-prediction based RL representation learning algorithms were partly inspired from the non-contrastive unsupervised learning algorithm \citep{grill2020bootstrap,chen2021exploring}.
A number of prior work attempt to understand the non-collapse learning dynamics of such algorithms, through the roles of semi-gradient and regularization \citep{tian2021understanding}, prediction head \citep{wen2022mechanism} and data augmentation \citep{wang2021towards}. 
Although our analysis is specialized to the RL case, the non-collapse mechanism (\cref{theorem:eigen}) is qualitatively different from prior work. Such a result is potentially useful in understanding the behavior of unsupervised learning as well.

\paragraph{RL representation learning via spectral decomposition.} One primary line of research in representation learning for RL is via spectral decomposition of the transition matrix $P^\pi$ or successor matrix $(I-\gamma P^\pi)^{-1}$. These methods are generally categorized as: (1) eigenvector-decomposition based approach, which typically assumes symmetry or real diagonizability of $P^\pi$ \citep{mahadevan2005proto,c.2018eigenoption,lyle2021effect}; (2) SVD-based approach, which is more generally applicable \citep{behzadian2019fast,ren2022spectral} and shows theoretical benefits to downstream RL tasks. Our work draws the connections between spectral decomposition and more empirically oriented RL algorithm, such as SPR \citep{schwarzer2021dataefficient}, PBL \citep{guo2020bootstrap} and BYOL-RL \citep{guo2022byol}, and is one step in the direction of formally chracterizing high performing representation learning algorithms. 

\paragraph{Forward-backward representations.} Closely related to bidirectional self-predictive learning is the forward-backward (FB) representations \citep{touati2021learning,blier2021learning}. By design, the forward representation learns value functions and backward representation learns visitation distributions. This design bears close connections to the left and right singular vectors of the transition matrix $P^\pi$, which bidirectional self-predictive learning seeks to approximate. Despite the high level connection, bidirectional self-predictive learning is purely based on self-prediction, and hence has much simpler algorithmic design.

\paragraph{Algorithms for PCA and SVD with gradient-based update.} The ODE systems in \cref{eq:uniformD-ode,eq:double-ode} bear close connections to ODE systems used for studying gradient-based incremental algorithms for PCA and SVD of empirical covariance matrices in classical unsupervised learning. Example algorithms include Oja's subspace algorithm \citep{oja1985stochastic,oja1992principal} and its extension to SVD \citep{diamantaras1996principal,weingessel1997svd}. A primary historical motivation for such algorithms is that they entail computing top $k$ eigenvectors or singular vectors with incremental gradient-based updates. This echoes with the observation we make in this paper, that self-predictive learning dynamics can be understood as gradient-based spectral decomposition on the transition matrix.

\section{Deep RL implementation}
\label{sec:exp}

We considered the single representation self-predictive learning dynamics as an idealized theoretical framework that aims to capture some essential aspects of a number of existing deep RL representation learning algorithms \citep{schwarzer2021dataefficient,guo2020bootstrap}.
Importantly, our theoretical analysis suggests that we can get more expressive representations by leveraging the bidirectional self-predictive learning dynamics in \cref{sec:double}.

Inspired by the theoretical discussions, we introduce the \emph{deep bidirectional self-predictive learning} algorithm for representation learning for deep RL. We build the deep bidirectional self-predictive learning algorithm on top of the representation learning used in BYOL-Explore \citep{guo2022byol}. 
While BYOL-Explore uses the prediction loss as a signal to drive exploration, we do not use such exploration bonuses in this work and focus only on the effect of representation learning.

We now provide a concise summary of how BYOL-RL works and how it is adapted for bidirectional self-predictive learning. In general partially observable environments, BYOL-RL encodes a history of observations
$h_t=(f(o_s))_{s\leq t}$ into its latent representation $\Phi(h_t)\in\mathbb{R}^k$ through a convnet $f:\mathcal{O}\rightarrow\mathbb{R}^k$ and a  LSTM. Then, the algorithm constructs a multi-step forward prediction $
    p\left(\Phi(h_t),a_{t:t+n-1})\right) 
$ with an open loop LSTM $p:\mathbb{R}^k\times(\mathcal{A})^n\rightarrow\mathbb{R}^d$. This forward prediction is against a back-up target computed at the $n$-step forward future time step $f(o_{t+n})$. The bidirectional self-predictive learning algorithm hints at a backward latent self-prediction objective, i.e., predicting the past latent observations based on future representations. In a nutshell, for deep bidirectional self-predictive learning, we implement the backward prediction in the same way as forward prediction, but with a \emph{reversed} time axis. See \cref{appendix:byolrl} for more technical details on BYOL-RL and how it is adapted for backward predictions.

A related but different idea has been adopted in PBL \citep{guo2019efficient}, where they propose a reverse prediction that matches $\tilde{f}(o_{t+n})$ for some function $\tilde{f}:\mathcal{O}\rightarrow\mathbb{R}^k$ against the encoded representation $\Phi(h_{t+n})$. The design of PBL is motivated by partial observability, and does not carry out backward predictions over time. See \cref{appendix:byolrl} for details on PBL and how it differs significantly from deep bidirectional self-predictive learning.

\subsection{Experiments}
We compare the deep bidirectional self-predictive learning algorithm with BYOL-RL \citep{guo2020bootstrap}. BYOL-RL is built on V-MPO \citep{Song2020V-MPO:}, an actor-critic algorithm which shapes the representation using policy gradient, without explicit representation learning objectives. 

Our testbed is DMLab-30, a collection of 30 diverse partially observable cognitive tasks in the 3D DeepMind Lab \citep{beattie2016deepmind}. We consider the multi-task setup where the agent is required to solve all $30$ tasks simultaneously. Representation learning has proved highly valuable in such a setting \citep{guo2019efficient}. See \cref{fig:dmlab30-pergame-imp-rl} in \cref{appendix:exp} for results on how BYOL-RL significantly improves over baseline RL algorithm per each game.

\begin{figure}[t]
    \centering
    \includegraphics[keepaspectratio,width=.45\textwidth]{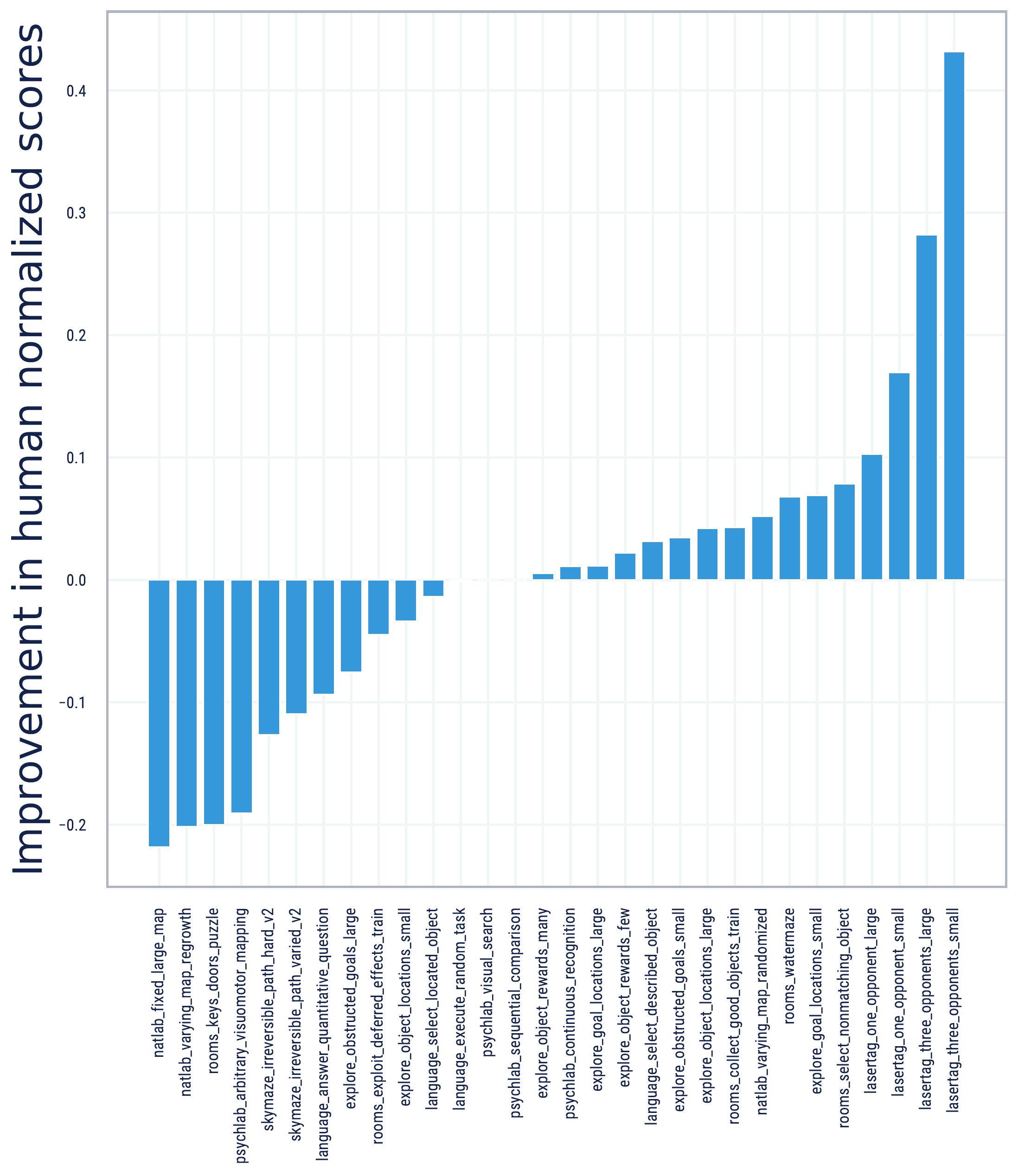}
     \caption{Per-game improvement of bidirectional self-predictive learning compared to baseline BYOL-RL, in terms of mean human normalized scores averaged across $3$ seeds. The scores are obtained at the end of training. Bidirectional self-predictive learning provides noticeable gains over BYOL-RL in certain cases, while it degrades performance in some others. See \cref{appendix:exp} for details.}
    \label{fig:dmlab30-pergame-imp}
\end{figure}

In \cref{fig:dmlab30-pergame-imp}, we show the per-game improvement of deep bidirectional self-predictive learning over BYOL-RL. In aggregate across all $30$ tasks, the two algorithms perform comparably as they either under-perform or over-perform each other in a similar number of games. We measure the performance of each task with the human normalized performance $(z_i-u_i)/(h_i-u_i)$ with $1\leq i\leq 30$, where $u_i$ and $h_i$ are the raw score performance of random policy and humans. A normalized score of $1$ indicates that the agent performs as well as humans on the task. Interestingly, there are a number of tasks on which backward predictions significantly improve over the baseline by as much as $0.4$ human normalized score. 
This comparison shows the promise of the deep bidirectional self-predictive learning algorithm when combined with deep RL agents.

\section{Conclusion}
In this work, we have presented a new way to understand the learning dynamics of self-predictive learning for RL. By identifying key algorithmic components to guarantee non-collapse, and drawing close connections between self-prediction and spectral decomposition of the transition dynamics, we have provided a first justification to the practical efficacy of a few empirically motivated algorithms. Our insights also naturally led to a novel representation learning algorithm, which mirrors SVD-based algorithms and enjoys more flexibility than self-prediction with a single representation. A deep RL instantiation of the algorithm has also proved empirically comparable to state-of-the-art performance.

Our results open up many interesting avenues for future research. A few examples: in hindsight, many assumption are due to the limitations of linear models; it will be of significant interest to study non-linear latent predictions and representations. 
Another interesting direction would be to study how self-prediction interacts with RL algorithms, such as TD-learning or policy gradient algorithms. Lastly, our analysis also provides insights to the unsupervised learning case, and hopefully motivates more investigation along this direction in the space.

\paragraph{Acknowledgements.} 
Many thanks to Jean-Bastien Grill and Csaba Szepesv\'ari for their feedback on an earlier draft of this paper. 
We are appreciative of the collaborative research environment within DeepMind.
The experiments in this paper were conducted using Python 3, and made heavy use of NumPy \citep{harris2020array}, SciPy \citep{2020SciPy-NMeth}, Matplotlib \citep{hunter2007matplotlib}, Reverb \citep{cassirer2021reverb}, JAX \citep{jax2018github} and the DeepMind JAX Ecosystem \citep{deepmind2020jax}.

\bibliography{main}
\bibliographystyle{plainnat}

\clearpage
\onecolumn

\begin{appendix}

\section*{\centering APPENDICES: Understanding Self-Predictive Learning for Reinforcement Learning}

\section{Detailed derivations of ODE systems }\label{appendix:ode}

We provide a derivation of the ODE systems in \cref{eq:uniformD-ode,eq:double-ode-2} below. We start with a few useful facts: recall that $x\sim d,y\sim P^\pi(\cdot|x)$ are one-hot encoding of states. Let $D$ be a diagonal matrix with $d$ its diagonal entries $D_{ii}=d_i,\forall 1\leq i\leq |\mathcal{X}|$. Then, we have the following properties:
\begin{align*}
    \mathbb{E}\left[xx^T\right] = D, \mathbb{E}\left[xy^T\right] = DP^\pi.
\end{align*}

\subsection{\cref{eq:uniformD-ode} for self-predictive learning}

Starting with \cref{eq:ode}, the first-order optimality condition for $P_t$ can be made more explicit
\begin{align*}
   \left(\Phi_t^T \mathbb{E}\left[xx^T\right] \Phi_t\right) P_t =  \Phi_t^T \mathbb{E}\left[xy^T\right]  \Phi_t \Rightarrow \left(\Phi_t^T D \Phi_t \right) P_t =  \Phi_t^T DP^\pi \Phi_t.
\end{align*}
We can expand the dynamics for $\Phi_t$ as follows,
\begin{align*}
    \dot{\Phi}_t = \left(D - D\Phi_t \left(\Phi^T D\Phi_t\right)^{-1}\Phi_t^T D \right) P^\pi \Phi_t (P_t)^T. 
\end{align*}

 Under \cref{assumption:D,assumption:init}, the above dynamics simplifies into
\begin{align*}
   P_t =  \Phi_t^T P^\pi \Phi_t,\ \ \dot{\Phi}_t=\left(I - \Phi_t\Phi_t^T\right) P^\pi \Phi_t (P_t)^T,
\end{align*}
which is the ODE system in \cref{eq:uniformD-ode}.

\subsection{\cref{eq:double-ode-2} for bidirectional self-predictive learning}

Since the bidirectional self-predictive learning dynamics introduces least square regression from $y$ to $x$, we need to calculate expectations such as $\mathbb{E}[yy^T]$ and $\mathbb{E}[yx^T]$. In general, it is challenging to express $\mathbb{E}[yy^T]$ as a function of $D$ and $P^\pi$. When $D$ is identity (\cref{assumption:D}) and when $P^\pi$ is doubly-stochastic (\cref{assumption:doublestochastic}), we have $D$ as a stationary distribution of $P^\pi$ and hence 
$\mathbb{E}[yy^T]=D$ and \begin{align*}
\mathbb{E}\left[yx^T\right]=\mathbb{E}\left[(xy^T)^T\right]=\left(\mathbb{E}[xy^T]\right)^T=\left(DP^\pi\right)^T = \left(P^\pi\right)^T D.
\end{align*}
From \cref{eq:double-update}, we can make explicit the form of the prediction matrix
\begin{align*}
    \left(\Phi_t^T \mathbb{E}\left[xx^T\right] \Phi_t\right) P_t &=  \Phi_t^T \mathbb{E}\left[xy^T\right]  \tilde{\Phi}_t \Rightarrow \left(\Phi_t^T D \Phi_t \right) P =  \Phi_t^T DP^\pi \tilde{\Phi}_t, \\ 
    \left(\tilde{\Phi}_t^T \mathbb{E}\left[yy^T\right] \tilde{\Phi}_t\right) \tilde{P}_t &=  \tilde{\Phi}_t^T \mathbb{E}\left[yx^T\right]  \Phi_t \Rightarrow \left(\tilde{\Phi}_t^T D \Phi_t \right) \tilde{P}_t =  \tilde{\Phi}_t^T \left(P^\pi\right)^TD \Phi_t,
\end{align*}
Next, we can expand the dynamics of $\Phi_t$ and $\tilde{\Phi}_t$ as follows
\begin{align*}
    \dot{\Phi}_t &= \left(D - D\Phi_t \left(\Phi^T D\Phi_t\right)^{-1}\Phi_t^T D \right) P^\pi \tilde{\Phi}_t (P_t)^T \\
    \dot{\tilde{\Phi}}_t &= \left(D - D\tilde{\Phi}_t \left(\tilde{\Phi}^T D\tilde{\Phi}_t\right)^{-1}\tilde{\Phi}_t^T D \right) \underbrace{D^{-1}\left(P^\pi\right)^T D}_{\tilde{P}^\pi} \Phi_t (\tilde{P}_t)^T. 
\end{align*}
Interestingly, $\tilde{P}^\pi$ is also a Markov transition matrix that corresponds to the reverse Markov chain. Finally, plugging into $D=I$ (\cref{assumption:D}) and thanks to \cref{assumption:init-double}, we recover the dynamics in \cref{eq:double-ode-2}
\begin{align*}
    \dot{\Phi}_t &= \left(I - \Phi_t\Phi_t^T  \right) P^\pi \tilde{\Phi}_t (P_t)^T \\
    \dot{\tilde{\Phi}}_t &= \left(I - \tilde{\Phi}_t\tilde{\Phi}_t^T \right) \left(P^\pi\right)^T \Phi_t (\tilde{P}_t)^T. 
\end{align*}

\subsection{Equivalence between assumptions in deriving \cref{eq:double-ode-2}}
Now we provide a discussion on the equivalence between assumptions in deriving the ODE for bidirectional self-predictive learning dynamics. An alternative assumption to \cref{assumption:doublestochastic} is
\begin{restatable}{assumption}{assumptionnextstateuniform}\label{assumptionnextstateuniform}  Given the sampling process $x\sim d, y\sim P^\pi(\cdot|x)$, the marginal distribution over next state $y$ is uniform.
\end{restatable}
Our claim is that given the uniformity assumption on the first-state distribution \cref{assumption:D}, \cref{assumption:doublestochastic} and \cref{assumptionnextstateuniform} are equivalent. 
To see why, given \cref{assumption:doublestochastic}, it is straightforward to see that uniform distribution is a stationary distribution to $P^\pi$. Starting from the first-state distribution, which is uniform, the next-state distribution is also uniform, which proves the condition in \cref{assumptionnextstateuniform}. Now, given \cref{assumptionnextstateuniform}, we conclude the uniform distribution $u=|\mathcal{X}|^{-1}1_{|\mathcal{X}|}$ is a stationary distribution to $P^\pi$. By definition of the stationary distribution, this means
\begin{align*}
    u^T P^\pi = u.
\end{align*}
The above implies that each column of $P^\pi$ sums to $1$, and so $P^\pi$ is doubly-stochastic (\cref{assumption:doublestochastic}).

\section{Proof of theoretical results}
\label{appendix:proof}

\theoremnocollapse*
\begin{proof}
Under the dynamics in \cref{eq:ode}, the prediction matrix $P_t$ optimally minimizes the loss function $L(\Phi_t,P_t)$ given the representation $\Phi_t$. 
Let $A_t=\Phi_tP_t\in\mathbb{R}^{|\mathcal{X}| \times k}$ be the matrix product. The chain rule combined with the first-order optimality condition on $P_t$ implies
\begin{align}
    \nabla_{P_t} L(\Phi_t,P_t) = \Phi_t^T \partial_{A_t} L(\Phi_t,P_t) = 0.\label{eq:optimality-cond}
\end{align}
On the other hand, the semi-gradient update for $\Phi_t$ can be written as 
\begin{align*}
    \dot{\Phi}_t = -\nabla_{\Phi_t} \mathbb{E}_{x\sim d, y \sim P^\pi(\cdot|x)} \left[\left\lVert P_t^T \Phi_t^Tx - \sg\left(\Phi_t^Ty\right) \right\rVert_2^2\right] = -\partial_{A_t}L(\Phi_t,P_t)(P_t)^T.
\end{align*}
Thanks to \cref{eq:optimality-cond}, we have 
\begin{align*}
    \Phi_t^T \dot{\Phi}_t = -\Phi_t^T \partial_{A_t}L(\Phi_t,P_t)(P_t)^T = 0.
\end{align*}
Then, taking time derivative on the covariance matrix
\begin{align*}
    \frac{d}{dt} \left(\Phi_t^T\Phi_t\right) = \dot{\Phi}_t^T\Phi_t + \Phi_t^T\dot{\Phi}_t = \left(\Phi_t^T\dot{\Phi}_t\right)^T + \Phi_t^T\dot{\Phi}_t = 0,
\end{align*}
which implies that the covariance matrix is constant.
\end{proof}

\corollarynocollapse*
\begin{proof}
Take any two representation vectors $\phi_{i,t}$ and $\phi_{j,t}$ with $i\neq j$, which at initialization are different. This implies the cosine similarity $\left\langle\phi_{i,0},\phi_{j,0}\right\rangle\neq 1$. Since under the dynamics in \cref{eq:ode}, the covariance matrix $\Phi_t^T\Phi_t$ is preserved, this means $\phi_{i,t}^T\phi_{j,t}, \phi_{i,t}^T\phi_{1,t}$ and $\phi_{i,t}^T\phi_{j,t}$ are all constants over time, which implies
\begin{align*}
    \left\langle\phi_{i,t},\phi_{j,t}\right\rangle = \left\langle\phi_{i,0},\phi_{j,0}\right\rangle\neq 1.
\end{align*}
This means the two vectors cannot be aligned along the same direction for all time $t\geq 0$.
\end{proof}

\lemmacritical*
\begin{proof}
Without loss of generality, consider the subset of first $k$ right eigenvectors $U=(u_1...u_k)$. 
Then $P^\pi U = U\Lambda$ for some diagonal matrix $\Lambda=\diag(\lambda_1...\lambda_k)$, where $\lambda_i$ is the eigenvalue corresponding to $u_i$.

If $\Phi_t = U$, then
\begin{align*}
    \dot{\Phi}_t=(I-UU^T) P^\pi U(P_t)^T = (I-UU^T) U\Lambda U(P_t)^T = 0.
\end{align*}

Next, for any set of $k$ orthonormal vectors with the same span as $U$, we can write them as $U'=UQ$ for some orthogonal matrix $Q\in\mathbb{R}^{k\times k}$. 
If $\Phi_t = U' = UQ$, then
\begin{align*}
    \dot{\Phi}_t=(I-UQQ^TU) P^\pi UQ(P_t)^T = (I-UU^T) U\Lambda UQ(P_t)^T = 0,
\end{align*}
which concludes the proof.
\end{proof}

\theoremeigen*
\begin{proof}
We first show that the objective is non-decreasing. We calculate
\begin{align*}
    \frac{d}{dt}f(\Phi_t) &= 4\cdot \text{Trace}\left(\left(\Phi_t^T P^\pi \Phi_t \right)^T \Phi_t^T P^\pi \dot{\Phi}_t\right) \\
    &=_{(a)} 4\cdot \text{Trace}\left(P_t \Phi_t^T P^\pi \left(I-\Phi_t\Phi_t^T\right) P^\pi \Phi_t P_t^T\right) \\
    &=_{(b)} 4\cdot \text{Trace}\left( \left( P^\pi \Phi_t P_t^T\right)^T \left(I-\Phi_t\Phi_t^T\right) P^\pi \Phi_t P_t^T\right), \\
\end{align*}
where (a) follows from $P_t=\Phi_t^T P^\pi \Phi_t$; (b) follows from the fact that $P^\pi$ is symmetric and as a result $P_t$ is symmetric. Now, let $A_t=P^\pi\Phi_t P_t^T$ and denote its column vectors as $A_t=[a_{1,t}...a_{k,t}]$. The above derivative rewrites as 
\begin{align*}
    4\cdot \sum_{i=1}^k a_{i,t}^T \left(I-\Phi_t\Phi_t^T\right) a_{i,t}.
\end{align*}
We remind that a projection matrix $M$ satisfies $M^2=M$ and $M^T=M$ and corresponds to an orthogonal projection onto certain subspace. 
Since $I-\Phi_t\Phi_t^T$ is a projection matrix, we have $a_{i,t}^T \left(I-\Phi_t\Phi_t^T\right) a_{i,t} \geq 0$ for any $a_{i,t}\in\mathbb{R}^\mathcal{X}$. Hence $\frac{d}{dt}f(\Phi_t)\geq 0$. Now, if $\Phi_t\notin\mathcal{C}_{P^\pi}$, this means there exists certain columns $a_{i,t}$ of $A_t$ such that $a_{i,t}\notin\spanop(\Phi_t)$. This means $a_{i,t}^T \left(I-\Phi_t\Phi_t^T\right) a_{i,t} > 0$ and therefore $\dot{f}>0$.

Finally, we examine the maximizer to $f(\Phi)$ under the constraint $\Phi^T\Phi=I_{k\times k}$. Since $\Phi^T P^\pi \Phi$ is symmetric, there exists an orthogonal matrix $Q$ such that 
\begin{align*}
    Q^T \Phi^T P^\pi \Phi Q = \Lambda,
\end{align*}
for some diagonal matrix $\Lambda$. Note that since $(\Phi Q)^T\Phi Q=I_{k\times k}$, it is equivalent to consider the optimization problem under a stronger constraint $\Phi^T\Phi=I_{k\times k}$ and $\Phi^T P^\pi \Phi=\Lambda$ for some diagonal matrix $\Lambda$. Therefore, the optimization problem becomes
\begin{align*}
    \max_{\Phi^T\Phi=I_{k\times k},\Phi^T P^\pi \Phi=\Lambda} \sum_{i=1}^k \Lambda_{ii}^2.
\end{align*}
Let $\Phi=[\phi_1,...\phi_k]$ with column vectors $\phi_i\in\mathbb{R}^{|\mathcal{X}|}$ for all $1\leq i\leq k$, then we have the equivalent optimization problem 
\begin{align*}
    \max_{\Phi^T\Phi=I_{k\times k},\Phi^T P^\pi \Phi=\Lambda} \sum_{i=1}^k \left(\phi_i^T P^\pi \phi_i\right)^2 &\leq_{(a)} \max_{\Phi^T\Phi=I_{k\times k},\Phi^T P^\pi \Phi=\Lambda} \sum_{i=1}^k \left\lVert P^\pi \phi_i\right\rVert_2^2 \\  &= \max_{\Phi^T\Phi=I_{k\times k},\Phi^T P^\pi \Phi=\Lambda} \sum_{i=1}^k \phi_i^T \left(P^\pi\right)^TP^\pi \phi_i.
\end{align*}
Here, (a) follows from the fact that $\phi_i$ is a unit-length vector and the application of the inequality $a^Tb\leq \left\lVert a\right\rVert_2 \left\lVert b\right\rVert_2 $. It is straightforward to see that the optimal solution to the last optimization problem is the set of eigenvectors of $P^\pi$ with top squared eigenvalues. Hence, 
\begin{align*}
    \max_{\Phi^T\Phi=I_{k\times k}} f(\Phi) \leq \sum_{i=1}^k \lambda_i^2.
\end{align*}
On the other hand, the $k$ eigenvectors of $P^\pi$ with top $k$ absolute eigenvalues is a feasible solution and therefore $\max_{\Phi^T\Phi=I_{k\times k}} f(\Phi) \geq \sum_{i=1}^k \lambda_i^2$. The above implies $\max_{\Phi^T\Phi=I_{k\times k}} f(\Phi) = \sum_{i=1}^k \lambda_i^2$, and the $k$ eigenvectors of $P^\pi$ with top $k$ absolute eigenvalues is a maximizer to the constrained optimization problem. It is then also clear that any $k$ orthonormal vectors with the same span as the top $k$ eigenvectors also achieves the maximum objective.
\end{proof}

\theoremnocollapsedouble*
\begin{proof}
We consider the forward and backward loss function separately. Following the arguments in the proof of \cref{theorem:nocollapse}, we see that since $P_t$ is computed as the optimal solution to $L_\mathrm{f}(P_t,\Phi_t)$, it satisfies the first-order optimality condition and as a result, $\Phi_t^T\dot{\Phi}_t=0$. This implies $\Phi_t^T\Phi_t$ is a constant matrix over time. Applying the same set of arguments to the backward loss function $L_\mathrm{b}(\tilde{\Phi}_t,\tilde{P}_t)$,  we conclude $\tilde{\Phi}_t^T\tilde{\Phi}_t$ is also a constant matrix over time.
\end{proof}

\lemmacriticaldouble* 
\begin{proof}
Without loss of generality, consider the subset of $k$ top singular vector pairs $U=(u_1...u_k),V=(v_1...v_k)$. By construction, they satisfy the equality $P^\pi V=U\Sigma$ and $\left(P^\pi\right)^TU=V\Sigma$ where $\Sigma=\diag(\sigma_1...\sigma_k)$ is the diagonal matrix with corresponding singular values.

Setting $\Phi_t=U,\tilde{\Phi}_t=V$,
we first verify the critical conditions for $\dot{\Phi}_t=0,\dot{\tilde{\Phi}}_t=0$:
\begin{align*}
    (I-\Phi_t\Phi_t^T) P^\pi \tilde{\Phi}_t (P_t)^T = (I-UU^T) P^\pi V V^T \left(P^\pi\right)^T U =_{(a)} (I-UU^T) U\Sigma^2 = 0.\\
    (I-\tilde{\Phi}_t\tilde{\Phi}_t^T) \left(P^\pi\right)^T \Phi_t (\tilde{P}_t)^T = (I-VV^T) \left(P^\pi\right)^T UU^T P^\pi V =_{(b)} (I-VV^T) V\Sigma^2 = 0.
\end{align*}
Here, (a) and (b) both follow from the property of the singular vector pairs. The above indicates that any $k$ singular vector pairs constitute a member of $\tilde{\mathcal{C}}_{P^\pi}$. 

Any orthonormal vectors with the same vector span as $U,V$ can be expressed as $UQ,VR$ for some orthogonal matrix $Q,R\in\mathbb{R}^{k\times k}$. Let $U'=UQ,V'=VR$, we verify the critical conditions when $\Phi_t=U',\tilde{\Phi}_t=V'$,
\begin{align*}
     (I-\Phi_t\Phi_t^T) P^\pi \tilde{\Phi}_t (P_t)^T = (I-U'(U')^T) P^\pi V' (V')^T \left(P^\pi\right)^T U' =_{(a)} (I-UU^T) U\Sigma^2 Q = 0.\\
     (I-\tilde{\Phi}_t\tilde{\Phi}_t^T) \left(P^\pi\right)^T \Phi_t (\tilde{P}_t)^T = (I-V'(V')^T) \left(P^\pi\right)^T U'(U')^T P^\pi V' =_{(b)} (I-VV^T) V\Sigma^2R = 0,
\end{align*}
where (a) and (b) follow from straightforward matrix operations. We have hence verified that any orthonormal vectors with the same vector span as any subset of $k$ singular vector pairs constitute a critical point.
\end{proof}

\theoremsvd*
\begin{proof}
We start by showing the SVD trace objective is non-decreasing on the ODE flow.
\begin{align*}
    \frac{d}{dt}\tilde{f}\left(\Phi_t,\tilde{\Phi}_t\right)  = 2\cdot \text{Trace}\left(\left(\Phi_t^T P^\pi \tilde{\Phi}_t\right)^T\left(\Phi_t^T P^\pi \dot{\tilde{\Phi}}_t\right)\right) + 2\cdot \text{Trace}\left(\left(\Phi_t^T P^\pi \tilde{\Phi}_t\right)^T\left(\dot{\Phi}_t^T P^\pi \tilde{\Phi}_t\right)\right).
\end{align*}
Examining the first term on the right above, plugging in the dynamics for $\dot{\tilde{\Phi}}$,
\begin{align*}
   \text{Trace}\left(\left(\Phi_t^T P^\pi \tilde{\Phi}_t\right)^T\left(\Phi_t^T P^\pi \dot{\tilde{\Phi}}_t\right)\right) &= \text{Trace}\left(\left(\Phi_t^T P^\pi \tilde{\Phi}_t\right)^T\Phi_t^T P^\pi \left(I-\tilde{\Phi}_t\tilde{\Phi}_t^T\right) \left(P^\pi\right)^T \Phi_t (\tilde{P}_t)^T  \right) \\
   &=_{(a)} \text{Trace}\left(\left(\left(P^\pi\right)^T \Phi_t (\tilde{P}_t)^T \right)^T \left(I-\tilde{\Phi}_t\tilde{\Phi}_t^T\right) \left(P^\pi\right)^T \Phi_t (\tilde{P}_t)^T.  \right)
\end{align*}
Here, (a) follows from the form of the prediction matrix $\tilde{P}_t = \tilde{\Phi}_t^T \left(P^\pi\right)^T\Phi_t$. Now, define $\tilde{A}_t=[a_{1,t}...a_{k,t}]\coloneqq\left(P^\pi\right)^T \Phi_t (\tilde{P}_t)^T$, the above rewrites as 
\begin{align*}
    \text{Trace}(A_t^T \left(I-\tilde{\Phi}_t\tilde{\Phi}_t^T\right) A_t)= \sum_{i=1}^k a_{i,t}^T \left(I-\tilde{\Phi}_t\tilde{\Phi}_t^T\right) a_{i,t}.
\end{align*}
Since $\left(I-\tilde{\Phi}_t\tilde{\Phi}_t^T\right)$ is an orthogonal projection matrix, we conclude the above quantity is non-negative. Similarly, we can show that the second term on the right above is also non-negative, which concludes $\dot{\tilde{f}}\geq 0$.

Now, we assume $(\Phi_t,\tilde{\Phi}_t)\notin\mathcal{\tilde{C}}_{P^\pi}$. Without loss of generality, we assume in this case $\dot{\tilde{\Phi}}_t\neq 0$, which implies there exists certain column $i$ such that $a_{i,t}$ is not in the span of $\tilde{\Phi}_t$. This means $a_{i,t}^T \left(I-\tilde{\Phi}_t\tilde{\Phi}_t^T\right) a_{i,t} > 0$ and subsequently $\dot{\tilde{f}}> 0$.

Finally, we show that under the constraint $\Phi^T\Phi=\tilde{\Phi}^T\tilde{\Phi}=I$, the maximizer to $\tilde{f}(\Phi,\tilde{\Phi})$ is any two set of $k$ orthonormal vectors with the same span as the $k$ singular vector pairs of $P^\pi$ with top singular values.
In general, the matrix $\Phi^T P^\pi \tilde{\Phi}$ is not diagonal. Consider its SVD 
\begin{align*}
    \Phi^T P^\pi \tilde{\Phi} = \tilde{U}\tilde{\Sigma}\tilde{V}^T,
\end{align*}
then we have $\left(\Phi\tilde{U}\right)^T P^\pi \tilde{\Phi}\tilde{V}=\tilde{\Sigma} $. Consider the new representation variable $\Phi'=\Phi\tilde{U}$ and $\tilde{\Phi}'=\tilde{\Phi}\tilde{V}$ and note they also satisfy the orthonormal constraint. Under the new variable, the matrix $(\Phi')^T P^\pi \tilde{\Phi}'=\tilde{\Sigma}$ is diagonal. Since $f(\Phi,\tilde{\Phi})=f(\Phi\tilde{U},\tilde{\Phi}\tilde{V})$, it is equivalent to solve the optimization  problem under an additional diagonal constraint
\begin{align*}
    \max_{\Phi^T\Phi=\tilde{\Phi}^T\tilde{\Phi}=I} f(\Phi,\tilde{\Phi}) = \max_{\Phi^T\Phi=\tilde{\Phi}^T\tilde{\Phi}=I,\Phi^TP^\pi\tilde{\Phi}\ \text{diagonal}} f(\Phi,\tilde{\Phi})
\end{align*}
When $\Phi^TP^\pi\tilde{\Phi}$ is diagonal, the objective rewrites as 
\begin{align*}
   \sum_{i=1}^k \left(\phi_i^T P^\pi \tilde{\phi}_i\right)^2 \leq_{(a)} \sum_{i=1}^k \left\lVert P^\pi \tilde{\phi}_i \right\rVert_2^2 =  \sum_{i=1}^k \tilde{\phi}_i^T \left(P^\pi\right)^TP^\pi \tilde{\phi}_i,
\end{align*}
where (a) follows from the fact that $\phi_i$ is a unit-length vector and the application of the inequality $a^Tb\leq \left\lVert a\right\rVert_2 \left\lVert b\right\rVert_2 $. Now, under the constraint $\tilde{\Phi}^T\tilde{\Phi}=I$, the right hand side is upper bounded by the sum of squared top $k$ singular values of $P^\pi$: $\sum_{i=1}^k\sigma_i^2$. Let $f^\ast$ be the optimal objective of the original constrained problem. We have hence established $f^\ast\leq \sum_{i=1}^k\sigma_i^2$. On the other hand, if we let $\Phi,\tilde{\Phi}$ to be the top $k$ singular vector pairs, they satisfy the constraint and this shows $f^\ast\geq\sum_{i=1}^k\sigma_i^2 $. 

In summary, we have $f^\ast=\sum_{i=1}^k\sigma_i^2$ and the top $k$ singular vector pairs $U,V$ are the maximizer. Any $k$ orthonormal vectors with the same span as  top $k$ singular vector pairs can be expressed as $U'=UQ,V'=VR$ for some orthogonal matrix $Q,R\in\mathbb{R}^{k\times k}$. Since $(U')^TU'=(V')^TV'=I$ and $f(U,V)=f(U',V')=f^\ast$, they
are also the maximizer to the constrained problem.
\end{proof}

\section{Extension of non-collapse property to general loss function}
\label{appendix:loss}

Thus far, we have focused on the squared loss function for understanding the self-predictive learning dynamics,
\begin{align}
\begin{split} 
   P_t &\in \arg\min_P L(\Phi_t,P),  \\ \dot{\Phi}_t &= - \nabla_{\Phi_t} \mathbb{E} \left[\left\lVert P_t^T \Phi_t^Tx - \sg\left(\Phi_t^Ty\right) \right\rVert_2^2\right].
\end{split}
\end{align}
We can extend the result to a more general class of loss function $L(\Phi_t,P_t)$. Such a loss function needs to be computed as an expectation over a function $F$ of the product prediction and representation matrix $\Phi_tP_t$, used for computing the prediction; and another argument for the representation matrix $\Phi_t$, used for computing the target. More formally, we can write $L(\Phi_t,P_t)=F\left(P_t\Phi_t,\Phi_t\right)=\mathbb{E}\left[f\left(P_t^T\Phi_t^Tx, \Phi_t^Ty\right)\right]$ for some function $f:\mathbb{R}^k\times\mathbb{R}^k\rightarrow k$. For the least squared case, we have $f(a,b)=\left\lVert a-b\right\rVert_2^2$. We consider the self-predictive learning dynamics with such a general loss function,
\begin{align}
\begin{split}  \label{eq:ode-general}
   P_t &\in \arg\min_P F\left(\Phi_t P, \Phi_t\right),  \\ \dot{\Phi}_t &= - \nabla_{\Phi_t} F\left(\Phi_tP_t,\text{sg}\left(\Phi_t\right)\right).
\end{split}
\end{align}
We show that the non-collapse property also holds for the above dynamics.
\begin{restatable}{theorem}{theoremnocollapsegeneral}\label{theorem:nocollapsegeneral} Assume the general loss function $L(\Phi,P)$ is such that the minimizer to $L(\Phi,P)$ satisfies the first-order optimality condition, then
under the dynamics in \cref{eq:ode-general}, the covariance matrix $\Phi_t^T\Phi_t \in\mathbb{R}^{k\times k}$ is constant over time.
\end{restatable}
\begin{proof}
The proof follows closely from the proof of \cref{theorem:nocollapse}.  Let $A_t=\Phi_tP_t\in\mathbb{R}^{|\mathcal{X}| \times k}$ be the matrix product, the assumption implies
\begin{align}
    \nabla_{P_t} F(\Phi_tP_t,\Phi_t) = \Phi_t^T \partial_{A_t} F(\Phi_tP_t,\Phi_t) = 0.\label{eq:optimality-cond-general}
\end{align}
On the other hand, the semi-gradient update for $\Phi_t$ can be written as 
\begin{align*}
    \dot{\Phi}_t = -\nabla_{\Phi_t} F\left(\Phi_tP_t,\Phi_t\right) = -\partial_{A_t}L(\Phi_t,P_t)(P_t)^T.
\end{align*}
Thanks to \cref{eq:optimality-cond-general}, we have 
\begin{align*}
    \Phi_t^T \dot{\Phi}_t = -\Phi_t^T \partial_{A_t}F\left(\Phi_tP_t,\Phi_t\right)(P_t)^T = 0.
\end{align*}
Then, taking time derivative on the covariance matrix
\begin{align*}
    \frac{d}{dt} \left(\Phi_t^T\Phi_t\right) = \dot{\Phi}_t^T\Phi_t + \Phi_t^T\dot{\Phi}_t = \left(\Phi_t^T\dot{\Phi}_t\right)^T + \Phi_t^T\dot{\Phi}_t = 0,
\end{align*}
which implies that the covariance matrix is constant.
\end{proof}
Notable examples of loss functions that satisfy the above assumptions include $L_1$ loss $f(a,b)=|a-b|$ and the regularized cosine similarity loss
$
    f(a,b) = -a^Tb / \left(\left\lVert a\right\rVert_2\left\lVert b\right\rVert_2+\epsilon\right)$, 
where $\epsilon>0$ is a regularization constant. 

\section{Discussions on critical points to the self-predictive learning dynamics}
\label{appendix:critical}

We provide further discussions on the critical points to the self-predictive learning dynamics in \cref{eq:uniformD-ode}. For convenience, we recall the set of ODEs
\begin{align*}
    P_t = \Phi_t^T P^\pi\Phi_t,  \ \ \dot{\Phi}_t = \left(I-\Phi_t\Phi_t^T\right) P^\pi \Phi_t (P_t)^T.
\end{align*}
According to \cref{lemmacritical}, under the assumption that $P^\pi$ is real diagonizable, any matrix with orthonormal columns with the same span as a set of $k$ eigenvectors constitutes a critical point to the ODE. Let $\mathcal{C}$ be the set of such matrices and recall that $\mathcal{C}_{P^\pi}$ is the set of critical points, we have $\mathcal{C}\subseteq\mathcal{C}_{P^\pi}$. 

We now consider two symmetric transition matrices, under which we have $\mathcal{C}=\mathcal{C}_{P^\pi}$ and $\mathcal{C}\subsetneq\mathcal{C}_{P^\pi}$ respectively. In both cases, we have $|\mathcal{X}|=2$ and $k=1$ for simplicity. In this case, $\Phi$ can be expressed as a column vector $\Phi_t$ is a $2$-dimensional vector and the prediction matrix $P_t$ is now a scalar. 

\paragraph{Case I where $\mathcal{C}=\mathcal{C}_{P^\pi}$.}
Consider the following transition matrix
\begin{align*}
   P^\pi = 
\begin{bmatrix}
    0.9 & 0.1\\
    0.1 & 0.9\\
\end{bmatrix}
\end{align*}
The transition matrix is symmetric and has eigenvalue $\lambda_1=1,\lambda_2=0.8$.
Note that
\begin{align*}
   U = [u_1,u_2] = 
\begin{bmatrix}
    \frac{1}{\sqrt{2}} &  \frac{1}{\sqrt{2}}\\
     \frac{1}{\sqrt{2}} & - \frac{1}{\sqrt{2}}\\
\end{bmatrix}
\end{align*}
is the matrix of eigenvectors. For convenience, we write
\begin{align*}
    \Phi_t = U \begin{bmatrix}
   \alpha_t\\
   \beta_t\\
\end{bmatrix}
\end{align*}
with coefficients $\alpha_t,\beta_t\in\mathbb{R}$. \cref{assumption:init} dictates $\alpha_t^2+\beta_t^2=1$. Now, we can calculate
\begin{align*}
    P_t = \Phi_t^T P^\pi \Phi_t = \alpha_t^2 + 0.8\beta_t^2 > 0,
\end{align*}
which is strictly positive and hence always non-zero.
Letting $\dot{\Phi}_t=0$, since $P_t\neq 0$, we conclude $
    \spanop(P^\pi \Phi_t) \subset \spanop(\Phi_t)$. This means $\Phi_t$ is invariant and must be a direct sum of subspaces spanned by eigenvectors. In other words, we have $\mathcal{C}_{P^\pi}\subset \mathcal{C}$ and hence the two sets are in fact equal.
    
\paragraph{Case II where $\mathcal{C}\subsetneq\mathcal{C}_{P^\pi}$.} Consider the following transition matrix we mentioned in \cref{eq:example1}
\begin{align*}
   P^\pi = 
\begin{bmatrix}
    0.1 & 0.9\\
    0.9 & 0.1\\
\end{bmatrix}.
\end{align*}
The transition matrix is symmetric and has eigenvalue $\lambda_1=1,\lambda_2=-0.8$.
Note that
\begin{align*}
   U = [u_1,u_2] = 
\begin{bmatrix}
    \frac{1}{\sqrt{2}} &  \frac{1}{\sqrt{2}}\\
     \frac{1}{\sqrt{2}} & - \frac{1}{\sqrt{2}}\\
\end{bmatrix}
\end{align*}
is the matrix of eigenvectors. For convenience, we write
\begin{align*}
    \Phi_t = U \begin{bmatrix}
   \alpha_t\\
   \beta_t\\
\end{bmatrix}
\end{align*}
with coefficients $\alpha_t,\beta_t\in\mathbb{R}$. \cref{assumption:init} dictates $\alpha_t^2+\beta_t^2=1$.
As before, we calculate
\begin{align*}
    P_t = \Phi_t^T P^\pi \Phi_t = \alpha_t^2- 0.8\beta_t^2 > 0.
\end{align*}
Now, we can identify $\alpha_t=\pm\frac{1}{\sqrt{1.8}},\beta_t=\pm\frac{\sqrt{0.8}}{\sqrt{1.8}}$ as four non-eigenvector critical points. Indeed, since $P_t=0$, we have $\dot{\Phi}_t=0$. However, since this critical point is a combination of two eigenvectors $u_1,u_2$, it does not span the same subspace as either just $u_1$ or $u_2$. In other words, we have found a critical point which does not belong to the set $\mathcal{C}$ and this implies $\mathcal{C}\subsetneq\mathcal{C}_{P^\pi}$.

\paragraph{Convergence of the dynamics in Case II.}
We replicate the diagram \cref{fig:criticalpoints} here in \cref{fig:criticalpoints-appendix}, where we graph critical points to the self-predictive learning dynamics with the above transition matrix on a unit circle. Recall that we have $k=1$ so that representations $\Phi_t$ are $2$-d vectors.
In addition to the eigenvector critical points plotted as blue dots (\cref{lemmacritical}), we have also identified other four critical points in red, corresponding to $\alpha_t=\pm\frac{1}{\sqrt{1.8}},\beta_t=\pm\frac{\sqrt{0.8}}{\sqrt{1.8}}$ in four quadrants of the $2$-d plane. 

The only local update dynamics consistent with the local improvement property (\cref{theorem:eigen}) is shown as black arrows on the unit circle. When the representation is initialized near the bottom eigenvector, it will converge to one of the four non-eigenvector critical points and not the top eigenvector.

\begin{figure}[t]
    \centering
    \includegraphics[keepaspectratio,width=.45\textwidth]{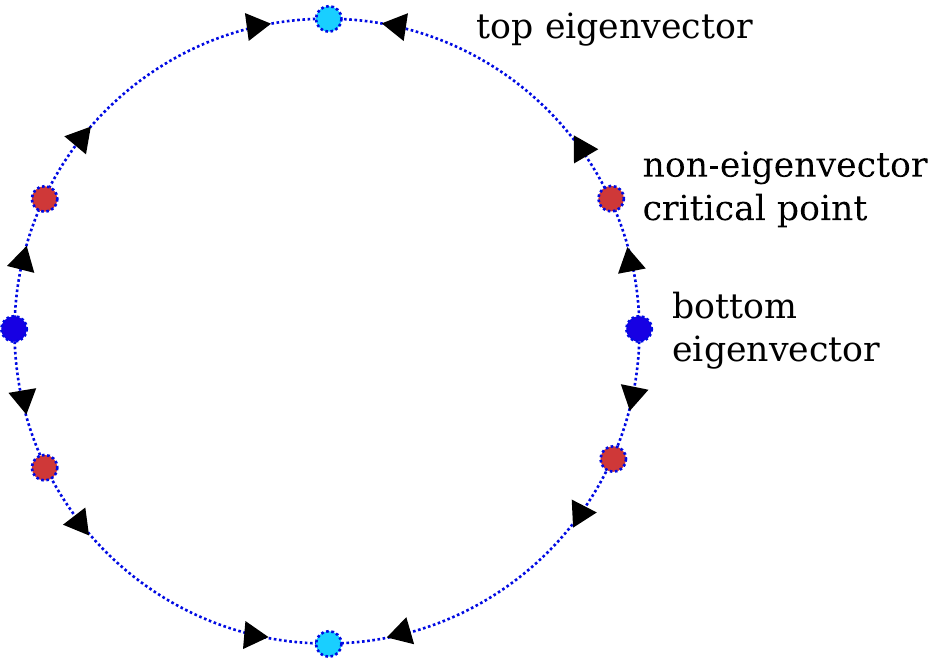}
    \caption{Critical points and local dynamics of the example MDP in \cref{eq:example1}. We consider $k=1$ so representations $\Phi_t$ are $2$-d vectors. There are four eigenvector critical points (light and dark blue) and four non-eigenvector critical points (red) of the ODE, shown on the unit circle. The black arrows show the local update direction based on the ODE. Initialized near the bottom eigenvector, the dynamics converges to one of the four non-eigenvector critical points and not to the top eigenvector.
   } 
    \label{fig:criticalpoints-appendix}
\end{figure}

\section{Algorithmic and implementation details on BYOL-RL}
\label{appendix:byolrl}

We provide background details on BYOL-RL \citep{guo2022byol}, a deep RL agent based on which our deep bidirectional self-predictive learning algorithm is implemented. We start with a relatively high-level description of the algorithm.

BYOL-RL is designed to work in the POMDP setting, where the agent observes a sequence of observations over time $o_t \in \mathcal{O}$. Define history $h_t \in\mathcal{H}$ at time $t$ as the combination of previous observations $h_t=(o_s)_{s\leq t}$ (note here the observation $o_s$ can contain action $a_{s-1}$). BYOL-RL adopts a few functions to represent the raw observation and history
\begin{itemize}
    \item An observation embedding function $f:\mathcal{O}\rightarrow\mathbb{R}^d$ which maps the observation $o_t$ into $d$-dimensional embeddings. In practice, this is implemented as a convolutional neural network.
    \item A recurrent embedding function $g:\mathbb{R}^k\times\mathbb{R}^d\rightarrow\mathbb{R}^k$ which processes the observation in a recurrent way. In practice, this is implemented as the core output of a LSTM.
\end{itemize}
In POMDP, we can consider the history $h_t$ as a proxy to the state in the MDP case. Let $\Phi(h_t)\in\mathbb{R}^k$ be the $k$-dimensional representation of history, we use the recurrent function to embed the history recursively
\begin{align*}
    \Phi(h_t) = g\left(\Phi(h_{t-1}), f(o_t)\right).
\end{align*}
BYOL-RL parameterizes the latent prediction function
$p:\mathbb{R}^k\times\left(\mathcal{A}\right)^n\rightarrow\mathbb{R}^k$, which can be understood as predicting the history representation $n$-step from now on, using only intermediate action sequence $a_{t:t+n-1}$
\begin{align*}
    p\left(\Phi(h_t),a_{t:t+n-1}\right) \in \mathbb{R}^k.
\end{align*}
Finally, another projection function $q: \mathbb{R}^k\rightarrow\mathbb{R}^d$ maps the predicted history representation, into the $d$-dimensional embedding space of the observation. Overall, the prediction objective is 
\begin{align*}
   \mathbb{E}\left[ \left\lVert q\left( p\left(\Phi(h_t),a_{t:t+n-1}\right)\right) - \text{sg}\left(f(o_{t+n})\right) \right\rVert_2^2 \right].
\end{align*}
The notation $\text{sg}$ indicates stop-gradient on the prediction target. All parameterized functions in BYOL-RL $f,g,p,q$, are optimized via semi-gradient descent.

Note that there are a number of discrepancies between theory and practice, such as multi-step prediction, action-conditional prediction and partial observability. See \cref{sec:discussion} for some discussions on possible extensions of the theoretical model to the more general case. We refer readers to the original paper \citep{guo2022byol} for detailed description of the neural network architecture and hyper-parameters.

\subsection{Details on deep bidirectional self-predictive learning with BYOL-RL}

We build the deep bidirectional self-predictive learning algorithm on top of BYOL-RL. The bidirectional self-predictive learning dynamics motivates a backward prediction loss function, which we instantiate as follows in the POMDP case.

For the backward prediction, we can instantiate an observation embedding function $\tilde{f}:\mathcal{O}\rightarrow\mathbb{R}^d$ and recurrent embedding function $\tilde{g}:\mathbb{R}^k\times\mathbb{R}^d\rightarrow\mathbb{R}^k$, analogous to the forward prediction case. We also parameterize a backward latent dynamics function $\tilde{p}:\mathbb{R}^k\times (\mathcal{A})^n\rightarrow\mathbb{R}^k$. Finally, we parameterize a projection function $\tilde{q}:\mathbb{R}^k\rightarrow\mathbb{R}^d$. The overall backward prediction objective is
\begin{align*}
    \mathbb{E}\left[\left\lVert \tilde{q}\left(\tilde{p}\left(\tilde{\Phi}(h_{t+n}),a_{t:t+n-1}\right)\right) - \text{sg}\left(\tilde{f}(o_t)\right)\right\rVert_2^2\right],
\end{align*}
where the backward recurrent representation is computed recursively as $\tilde{\Phi}(h_{t-1}) = \tilde{g}\left(\tilde{\Phi}(h_t) , \tilde{f}(o_{t-1})\right)$. In other words, we can understand the backward prediction problem as almost exactly mirroring the forward prediction problem.

As a design choice, we share the observation embedding in both the forward and backward process $f=\tilde{f}$. The motivation for such a design choice is that one arguably expects the observation embedding to share many common features for both the forward and backward process, when the input observations are images; secondly, since the embedding function is usually a much larger network compared to rest of the architecture, parameter sharing helps reduce the computational cost.

\subsection{Differences from PBL \citep{guo2019efficient}}

PBL is a representation learning algorithm based on both forward and backward predictions. In the POMDP case, PBL simply parameterizes a projection function $\tilde{q}_\text{pbl}:\mathbb{R}^k\rightarrow\mathbb{R}^d$. The backward prediction loss is computed as
\begin{align*}
    \mathbb{E}\left[\left\lVert \tilde{q}_\text{pbl}\left(\tilde{f}\left(o_{t+n}\right)\right) - \text{sg}\left(\Phi(h_{t+n})\right) \right\rVert_2^2\right].
\end{align*}
In other words, the backward prediction seeks to predict the recurrent history embedding $\Phi(h_{t+n})$ from the observation embedding $\tilde{f}(o_{t+n})$. By design, the backward prediction shares the same observation embedding function as the forward prediction $\tilde{f}=f$.

\section{Transition matrix to illustrate the failure mode of self-predictive learning}
\label{appendix:double-single}

To design examples that illustrate the failure mode of single representation learning dynamics, it is useful to review the intuitive interpretations of left and right singular vectors of $P^\pi$ as clustering states with certain similar features. Left singular vectors cluster together states with similar outgoing distribution, i.e., states with similar rows in $P^\pi$. Meanwhile, right singular vectors cluster together states with similar incoming distributions, i.e., states with similar columns in $P^\pi$. 

The example transition matrix with $|\mathcal{X}|=3$ states is
\begin{align*}
   P^\pi = 
\begin{bmatrix}
    0 & 1/2 & 1/2 \\
    0 & 1/2 & 1/2 \\
    1 & 0 & 0 \\
\end{bmatrix}
\end{align*}
We can calculate the top-1 left and right singular vectors as
\begin{align*}
    u_0 = \left[-\frac{1}{\sqrt{2}}, -\frac{1}{\sqrt{2}}, 0\right], \tilde{u}_0 = \left[0, -\frac{1}{\sqrt{2}}, -\frac{1}{\sqrt{2}}\right],
\end{align*}
which concides with the previous intuition that the top left singular vector should cluster together the first two states. Indeed, by assigning a value of $-1/\sqrt{2}$ to the first two states, the top left singular vector effectively considers the first two states as being identical. Meanwhile,
the top right singular vector should cluster together the last two states.

\section{Experiment details}
\label{appendix:exp}

We provide additional details on the experimental setups in the paper.

\subsection{Tabular experiments}

Throughout, tabular experiments are carried out on randomly generated MDPs. Instead of explicitly generating the MDPs, we generate the state transition matrix $P^\pi\in\mathbb{R}^{|\mathcal{X}\times|\mathcal{X}|}$ where $|\mathcal{X}|=20$ by default. In general, the algorithm learns $k=2$ representation columns.

\paragraph{Generating random $P^\pi$.}
To generate random doubly-stochastic transition matrix, we start by randomly initialize entries of $P$ to be i.i.d. $\text{Uniform}(0,1)$. Then we carry out column normalization and row normalization 
\begin{align*}
    P_{ij} \leftarrow P_{ij} / \sum_k P_{ik} ,
    P_{ij} \leftarrow P_{ij} / \sum_k P_{kj}
\end{align*}
until convergence. It is guaranteed that $P$ has row sum and column sum to be both $1$, and is hence doubly-stochastic. The transition matrix is computed as
\begin{align*}
    P^\pi = \alpha P + (1-\alpha) P_\text{perm}, 
\end{align*}
where $P_\text{perm}$ is a randomly generated permutation matrix and $\alpha\sim \text{Uniform}(0,1)$. It is straightforward to verify $P^\pi$ is doubly-stochastic. To generate symmetric transition matrix, we follow the above procedure and apply a symmetric operation $(P^\pi + \left(P^\pi\right)^T)/2$, which produces $P^\pi$ as a symmetric transition matrix. In \cref{fig:violation}, we generate doubly-stochastic matrices as examples of non-symmetric matrices.

\paragraph{Normalized trace objective.}
When plotting trace objectives, we usually calculate the normalized objective. For symmetric matrix $P^\pi$, the normalizer is the sum of top-$k$ squared eigenvalue of $P^\pi$: let $\lambda_i$ be the eigenvalues of $P^\pi$ ordered such that $|\lambda_i|\geq|\lambda_{i+1}|$, then normalizer is $\sum_{i=1}^k |\lambda_i|^2$. For double-stochastic matrix $P^\pi$, the normalizer is the squaresum of topd ma-$k$ximum singular singular value: let $\sigma_i$ be the singular values of $P^\pi$, the normalizer is $\sum_{i=1}^k \sigma_i^2$.  Such normalzers upper bound the trace objective and SVD trace objective respectively.

\paragraph{ODE vs. discretized update.} We carry out experiments in the tabular MDPs with setups. In \cref{fig:violation,fig:double_single}, we simulate the exact ODE dynamics using the Scipy ODE solver \citep{virtanen2020scipy}. In \cref{fig:collapse,fig:finitelr}, we simluate the discretized process using a finite learning rate $\eta$ on the representation matrix $\Phi_t$. This corresponds to implementing the update rule in \cref{eq:phi-update}. By default, in discretized updates we adopt $\eta=10^{-3}$ so that the non-coallpse property is almost satisfied. 

\subsubsection{Complete result to \cref{fig:violation}}

We present the complete result to \cref{fig:violation} in \cref{fig:violation-complete}, where we display the evolution of the trace objective for a total of $100$ iterations in \cref{fig:violation-complete}(a). Note that as the simulation runs longer, in the non-symmetric MDP case, we can observe very small oscillations in the trace objective. However, overall, the trace objective improves significantly compared to the initial values. 

In \cref{fig:collapse}(b), we make more clear the individual runs across different MDPs. Though in most MDPs, the trace objective is improved significantly given $100$ iterations, there exists MDP instances where the improvement is very slow and appear highly monotonic in the non-symmetric case. In the symmetric case, there can exist MDPs where the rate of improvement for the trace objective is slow too.

\begin{figure}[t]
    \centering
    \subfigure[Complete run]{\includegraphics[keepaspectratio,width=.45\textwidth]{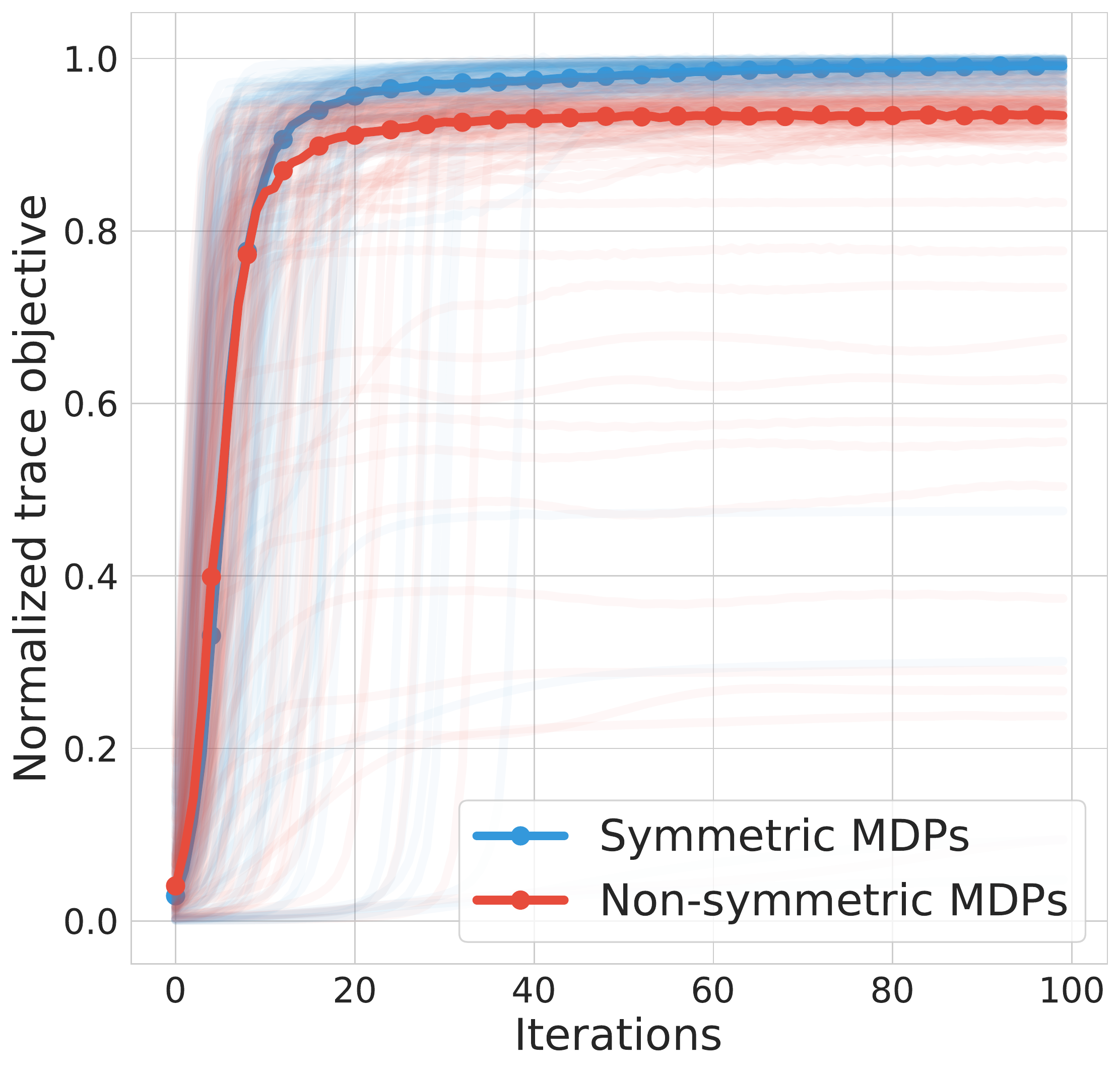}
     }
    \subfigure[Individual runs]{\includegraphics[keepaspectratio,width=.45\textwidth]{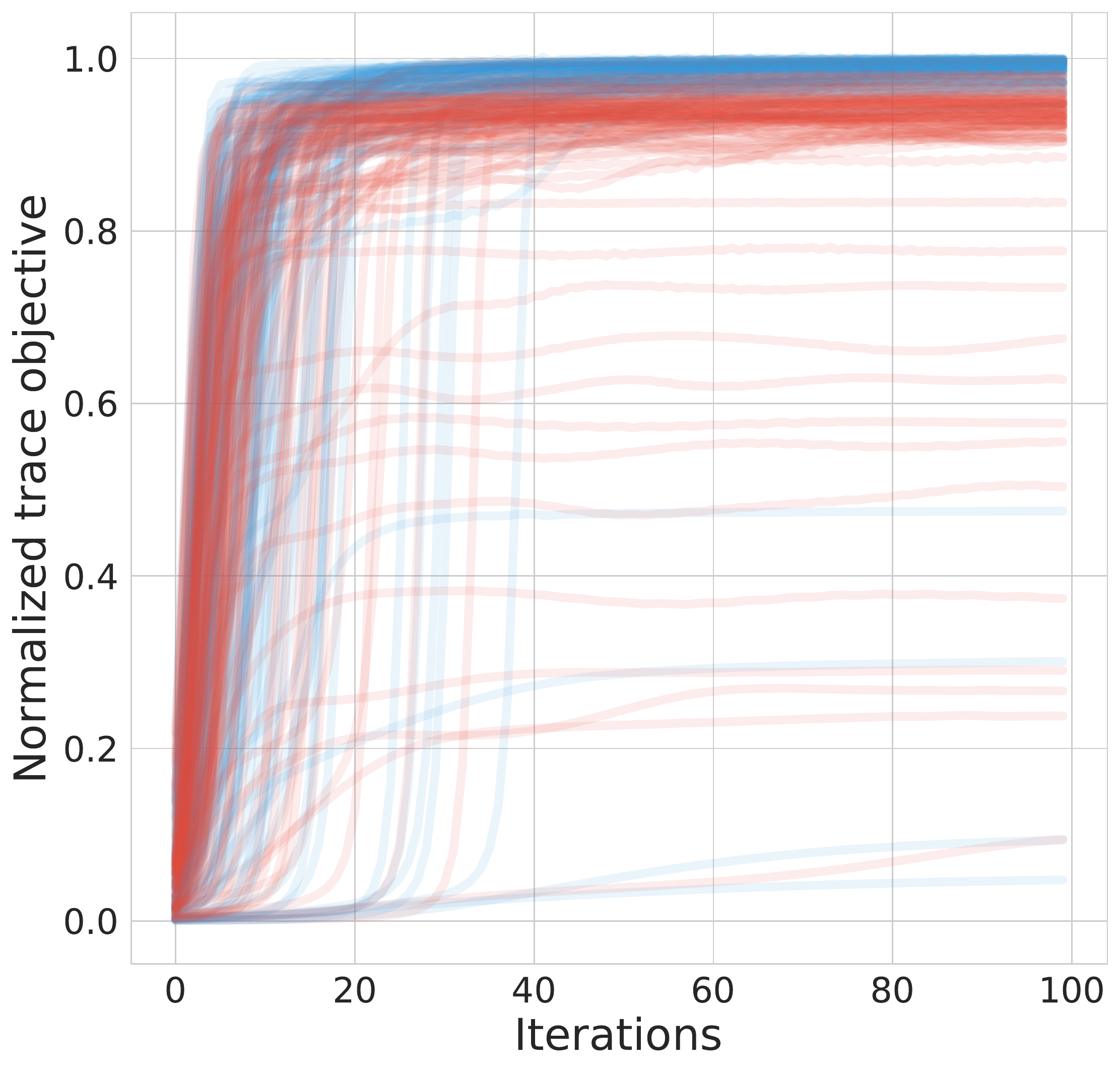}
         }
     \caption{Complete result for \cref{fig:violation} where we now show the result over a large number of iterations in (a). In (b), we make more clear the individual runs across different MDPs.   }
    \label{fig:violation-complete}
\end{figure}

\subsubsection{Effect of target network}

In practice, it is common to maintain a target network for constructing prediction targets for self-prediction \citep{guo2019efficient,schwarzer2021dataefficient,guo2020bootstrap}. Under our framework, the prediction loss function is
\begin{align*}
    \tilde{L}(\Phi,P,\Phi') = \mathbb{E}_{x\sim d, y \sim P^\pi(\cdot|x)} \left[\left\lVert P^T \Phi^Tx - (\Phi')^Ty \right\rVert_2^2\right],
\end{align*}
where $\Phi'$ is the target network, or target representation matrix. The target representation matrix is updated via moving average towards the online representation matrix
\begin{align*}
    \frac{d}{dt}\Phi_t' = \beta(\Phi_t-\Phi_t').
\end{align*}
The overall self-predictive learning dynamics with target representation is: 
\begin{align}
\begin{split} \label{eq:target-ode}
   P_t &= \arg\min_P \tilde{L}(\Phi_t,P,\Phi_t'),  \\
   \dot{\Phi}_t' &= \beta(\Phi_t-\Phi_t'),\\
   \dot{\Phi}_t &= - \nabla_{\Phi_t} \mathbb{E} \left[\left\lVert P_t^T \Phi_t^Tx - \sg\left((\Phi_t')^Ty\right) \right\rVert_2^2\right].
\end{split}
\end{align}

In \cref{fig:target}, we carry out ablation on the effect of $\beta$. We consider the symmetric MDP case where the self-predictive learning dynamics in \cref{eq:ode} should monotonically improve the trace objective. Here, we also plot the trace objective. When $\beta=0$, the result shows that the trace objective still improves compared to the initialization, though such an improvement is much more limited and can be very non-monotonic. When $\beta>0$, we see that the learning dynamics behaves very similarly to \cref{eq:ode}. 

\begin{figure}[t]
    \centering
    \includegraphics[keepaspectratio,width=.45\textwidth]{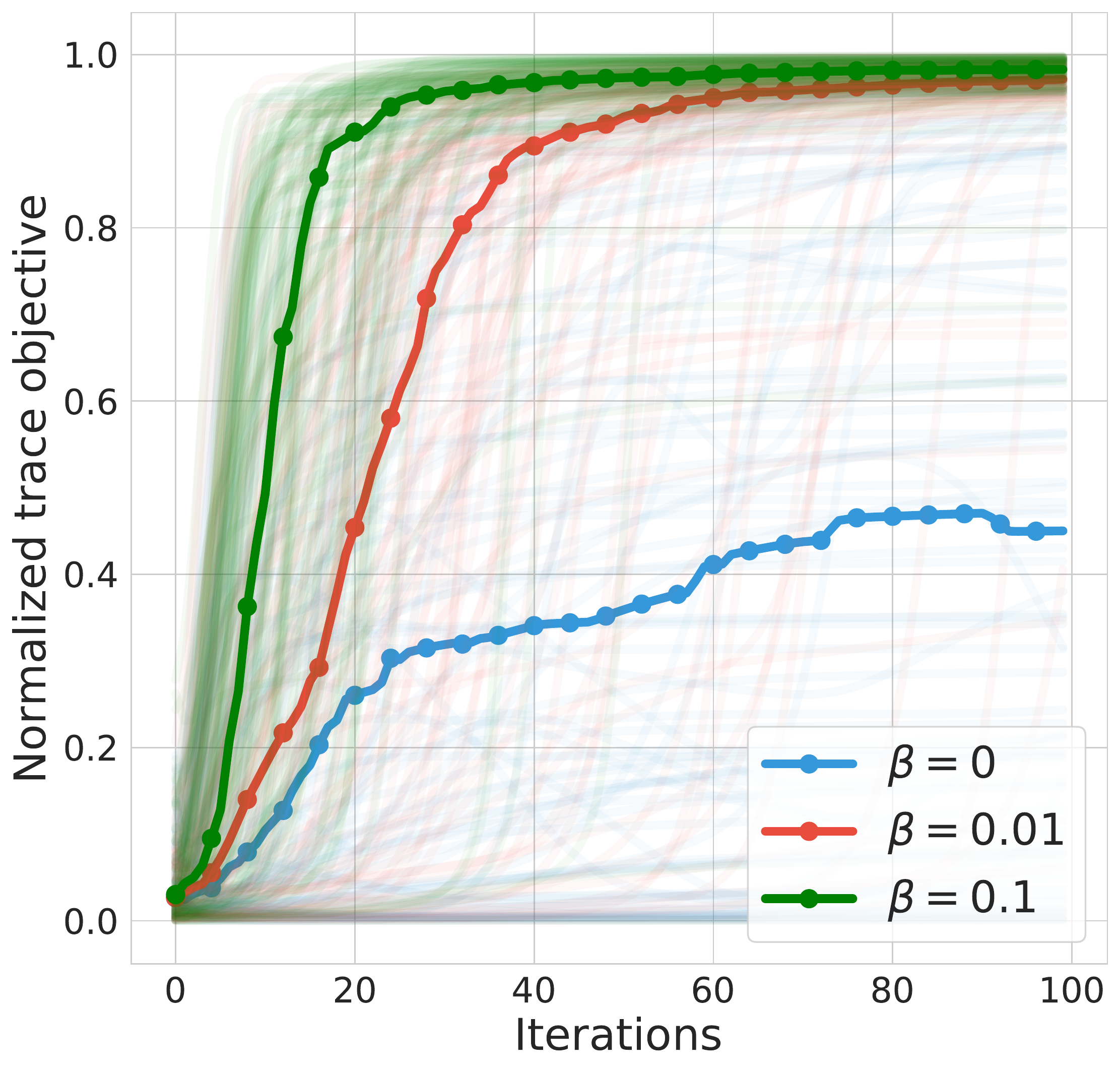}
     \caption{Impact of target representation matrix on the learning dynamics. We introduce a target representation matrix $\Phi_t'$ whose dynamics is $\frac{d}{dt}\Phi_t'=\beta(\Phi_t-\Phi_t')$, i.e., based on a moving average update towards the main representation matrix $\Phi_t$. With the target representation matrix in place, the trace objective still improves overall, but the improvement is likely to be non-monotonic.}
    \label{fig:target}
\end{figure}

\subsubsection{Ablation on how hyper-parameters impact non-collapse dynamics}

We now present results on how a number of different hyper-parameters impact the non-collapse dynamics: finite learning rate and non-optimal predictor.

\paragraph{Finite learning rate.}
Thus far, our theory has been focused on the continuous time case, which corresponds to a infinitesimally small learning rate. With finite learning rate, we expect the non-collapse to be violated. \cref{fig:finitelr}(a) shows the effect of finite learning rate on the preservation of the cosine similarity between two representation vectors $\phi_{1,t}$ and $\phi_{2,t}$. The two vectors are initialized to be orthogonal, so their cosine similarity is initialized at $0$. We consider a grid of learning rate $\eta\in\{0.01, 0.1, 1, 10\}$; to ensure fair comparison, for learning rate $\eta$, we showcase the cosine similarity $\left\langle \phi_{1,t}, \phi_{2,t}\right\rangle $ at iteration $T/\eta$ with $T=10000$. Interpreting $\eta$ as the magnitude of the step-size, this is to ensure that at each learning rate, the result is obtained after updating for a total step-size of $\eta \cdot T/\eta = T$.

As \cref{fig:finitelr}(a) shows, when the learning rate increases, the cosine similarity $\left\langle \phi_{1,t}, \phi_{2,t}\right\rangle $ increases, indicating a more severe violation of the non-collapse property. As $\left\langle \phi_{1,t}, \phi_{2,t}\right\rangle \rightarrow 1$, the two representation vectors become more and more aligned with each other, and eventually coallpse to the same direction.

\paragraph{Non-optimal predictor.} Our theory has suggested that the optimal predictor is important for the non-collapse of the self-predictive learning dynamics. To assess how sensitive the non-collapse property is to the level of \emph{imperfection} of the predictor, at each time step $t$, let $P_t^\ast$ be the optimal predictor. We set the prediction matrix as a corrupted version of the optimal prediction matrix,
\begin{align*}
    P_t = P_t^\ast + \epsilon,
\end{align*}
where $\epsilon\in\mathbb{R}^{k\times k}$ is a noise matrix whose entries are sampled i.i.d. from Gaussian distribution $\mathcal{N}(0,\sigma^2)$. Here, $\sigma$ determines the level of noise used for corrupting the predictor. This is meant to emulate a practical setup where the prediction matrix is not learned perfectly.

In \cref{fig:finitelr}(b), we show the cosine similarity between $\phi_{1,t}$ and $\phi_{2,t}$ after a fixed number of iterations $T=10000$. As the noise scale $\sigma$ increases, we observe an increasing tendency to collapse.

\begin{figure}[t]
    \centering
    \subfigure[Finite learning rate]{\includegraphics[keepaspectratio,width=.45\textwidth]{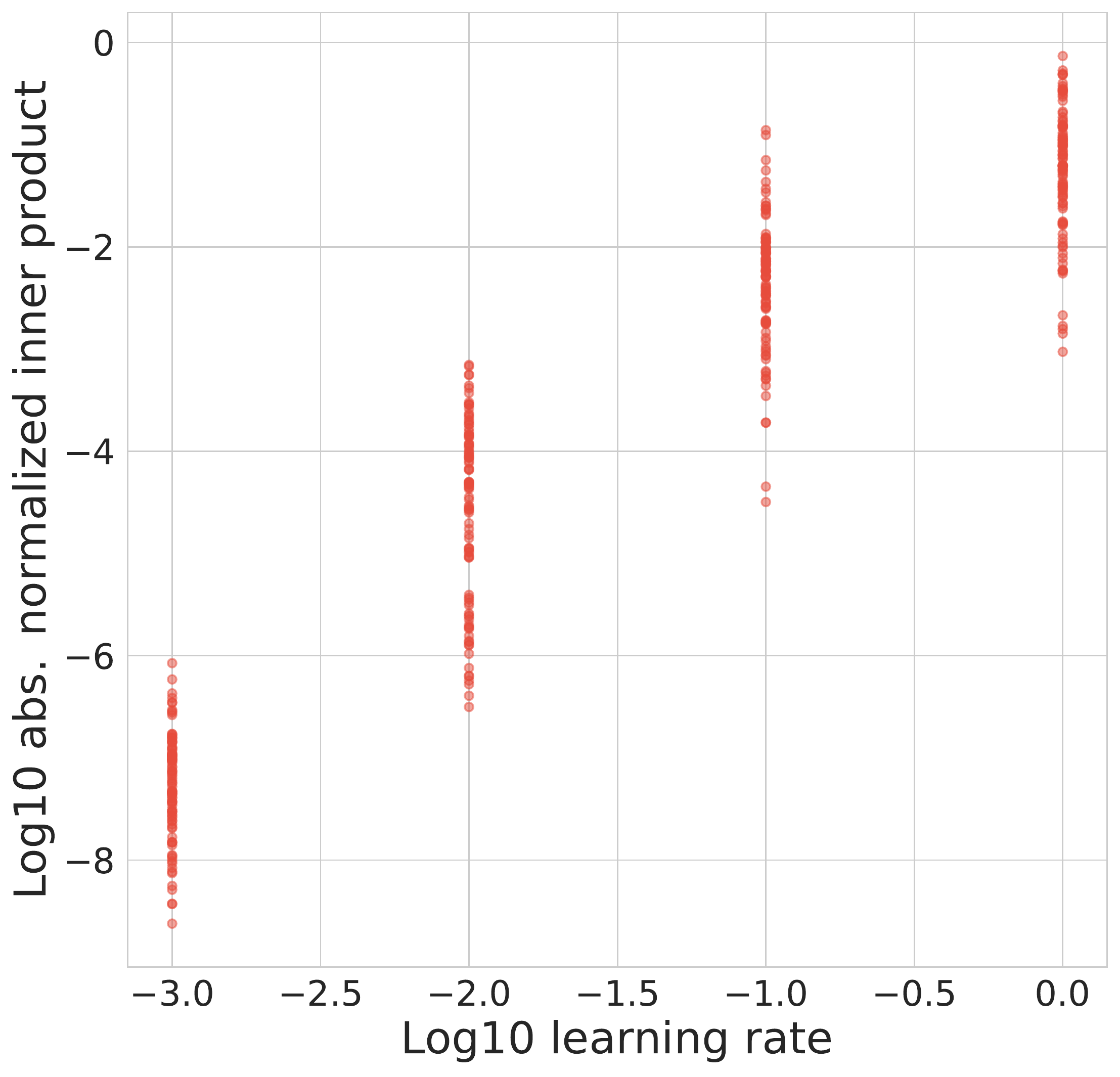}}
    \subfigure[Non-optimal predictor]{\includegraphics[keepaspectratio,width=.45\textwidth]{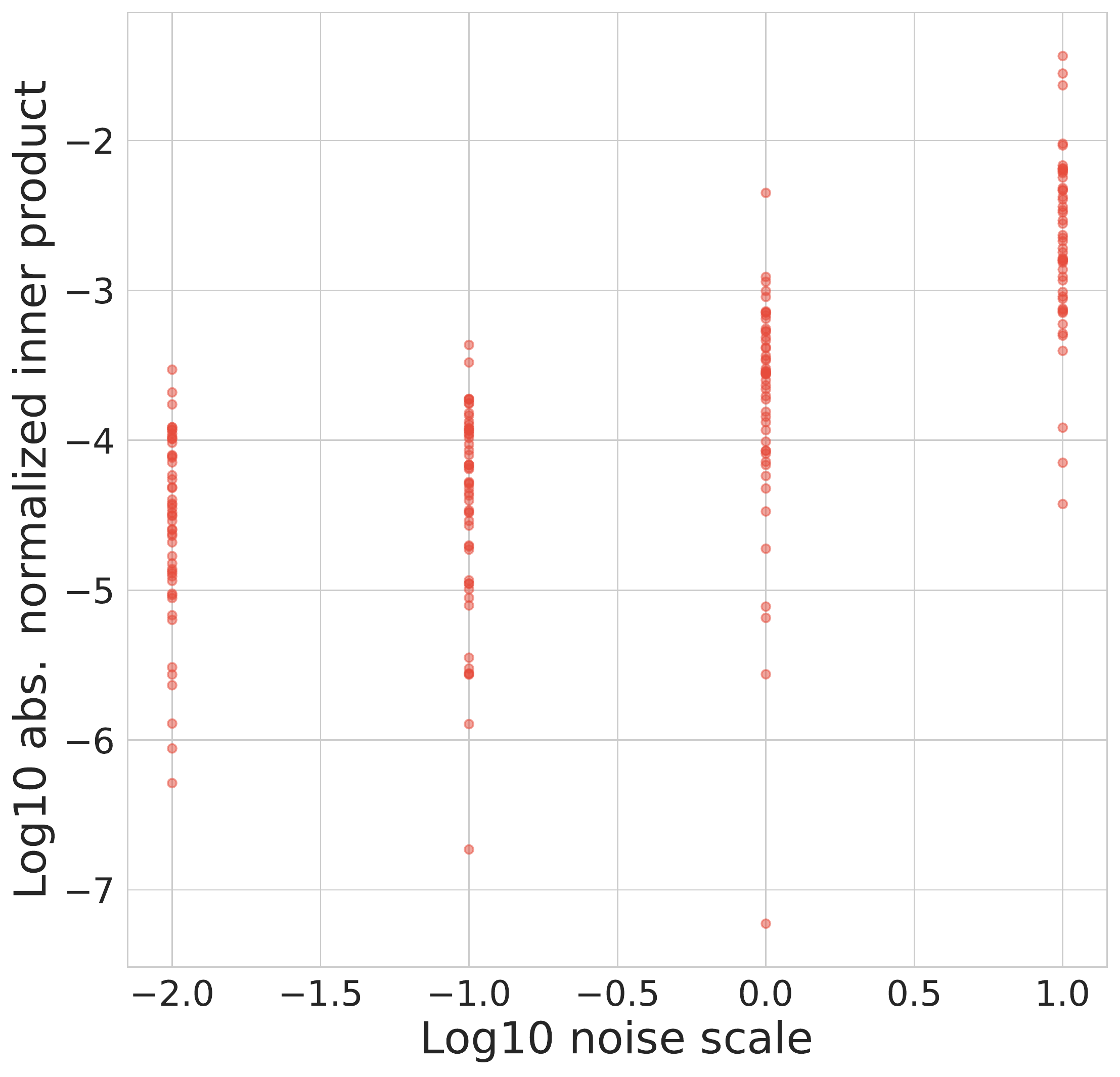}}
    \caption{Ablation experiments to assess the sensitivity of the non-collapse property to finite learning rate and non-optimal prediction matrix. Across all plots, $y$-axis shows the cosine similarity $\left\langle \phi_{1,t},\phi_{2,t}\right\rangle$ after some iterations of learning. Since the two vectors are initialized to be orthogonal, as the inner product increases from $0$ to $1$, we expect the representations to collapse to the same direction.} 
    \label{fig:finitelr}
\end{figure}

\subsection{Deep RL experiments}

We provide details on the deep RL experiments. 

We compare the deep bidirectional self-predictive learning algorithm, an algorithm inspired from the bidirectional self-predictive learning dynamics in \cref{eq:double-ode}, with BYOL-RL \citep{guo2020bootstrap}. BYOL-RL can be understood as an application of self-predictive learning dynamics in the POMDP case. BYOL-RL is built on V-MPO \citep{Song2020V-MPO:}, an actor-critic algorithm which shapes the representation using policy gradient, without explicit representation learning objectives. 
See \cref{appendix:byolrl} for a more detailed description of the BYOL-RL agent and how it is adapted to allow for bidirectional self-predictive learning. 

BYOL-RL implements a forward prediction loss function $L_\text{fwd}$, which is combined with V-MPO's RL loss function
\begin{align*}
    L_\text{BYOL-RL} = L_\text{rl} + L_\text{fwd}.
\end{align*}
The bidirectional self-predictive learning algorithm introduces a backward prediction loss function
\begin{align*}
    L_\text{bidirectional} = L_\text{rl} + L_\text{fwd} + \alpha L_\text{bwd},
\end{align*}
where $\alpha\geq 0$ is the only extra hyper-parameter we introduce. Throughout, we set $\alpha=1$ which introduces an equal weight between the forward and backward predictions, as this most strictly adheres to the theory. All hyper-parameters and architecture are shared across experiments wherever possible. We refer readers to Guo et al. \citep{guo2020bootstrap} for complete information on the network architecture and hyper-parameters.

\paragraph{Test bed.}
Our test bed is DMLab-30, a collection of 30 diverse partially observable cognitive tasks in the 3D DeepMind Lab \citep{beattie2016deepmind}. 
DMLab-30 has visual input $v_t$ to the agent, along with the agent's previous action $a_{t-1}$ and reward function $r_{t-1}$, form the observation at time $t$: $o_t=(v_t,a_{t-1},r_{t-1})$. 

To better illustrate the importance of representation learning, we consider the multi-task setup where the agent is required to solve all $30$ tasks simultaneously. In practice, at each episode, the agent uniformly samples an environment out of the $30$ tasks and generates a sequence of experience. Since the task id is not provided, the agent needs to implicitly infer the task while interacting with the sampled environment. This intuitively explains why representation learning is valuable in such a setting, as observed in prior work \citep{guo2019efficient}. 

In \cref{fig:dmlab30-pergame-imp-rl}, we compare the per-game performance of BYOL-RL with the RL baseline, measured in terms of the human normalized scores. Here, let $z_i$ be the raw score for the $i$-th game, $u_i$ the raw score of a random policy and $h_i$ the raw score of humans, then the human normalized score is calculated as $\frac{z_i-u_i}{h_i-u_i}$. Indeed, we see that BYOL-RL significantly out-performs the RL baseline across most games.

\begin{figure}[t]
    \centering
    \includegraphics[keepaspectratio,width=.45\textwidth]{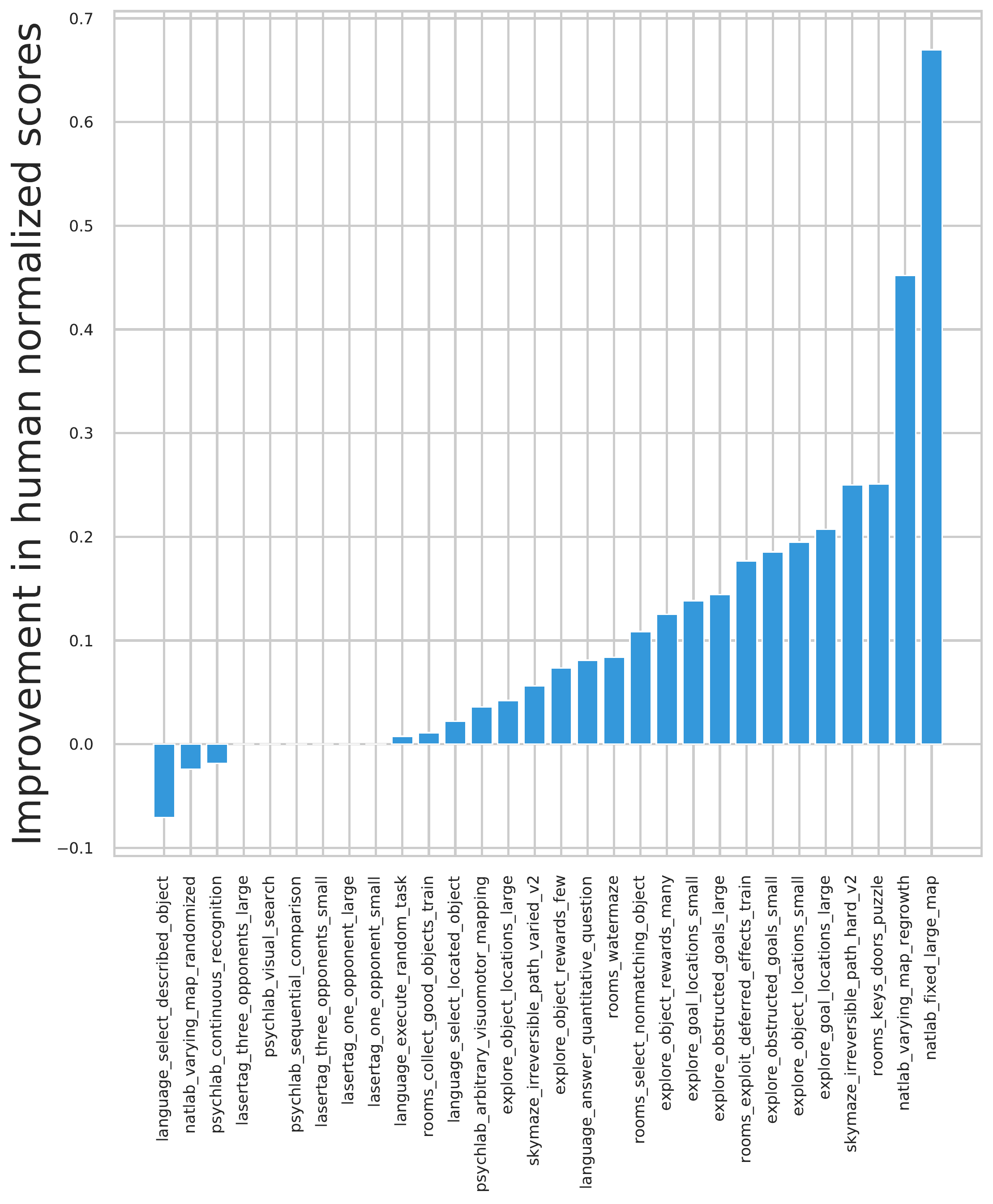}
     \caption{Per-game improvement of   BYOL-RL compared to baseline RL algorithm, in terms of mean human normalized scores averaged across $3$ seeds. The scores are obtained at the end of training. The improvement in performance is significant in most games.}
    \label{fig:dmlab30-pergame-imp-rl}
\end{figure}

\end{appendix}

\end{document}